\newif\ifarxiv

\documentclass[12pt]{colt2021} 
\arxivfalse

\ifarxiv
\title{Sequential prediction under log-loss and misspecification}
\else
\title[Sequential prediction under log-loss]{Sequential prediction under log-loss and misspecification}
\usepackage{times}
\fi

\ifarxiv
	\RequirePackage{amsthm,amsmath,amsfonts,amssymb}
	\RequirePackage[numbers]{natbib}
	\RequirePackage[colorlinks,citecolor=blue,urlcolor=blue]{hyperref}
	\RequirePackage{graphicx}
	\RequirePackage{color,xspace}
	\RequirePackage{ifpdf}
\fi

	\makeatletter
	\let\over=\@@over \let\overwithdelims=\@@overwithdelims
	\let\atop=\@@atop \let\atopwithdelims=\@@atopwithdelims
  	\let\above=\@@above \let\abovewithdelims=\@@abovewithdelims
  	\makeatother

\ifarxiv

	\newtheorem{theorem}{Theorem}[section]
	\newtheorem{lemma}[theorem]{Lemma}

	\theoremstyle{remark}
	\newtheorem{definition}[theorem]{Definition}

	\newtheorem{remark}{Remark}
\fi

\newcommand{\mreals}{\ensuremath{\mathbb{R}}}

\ifx\eqref\undefined
	\newcommand{\eqref}[1]{~(\ref{#1})}
\fi
\ifx\mod\undefined
	\def\mod{\mathop{\rm mod}}
\fi

\def\argmin{\mathop{\rm argmin}}
\def\argmax{\mathop{\rm argmax}}

\def\exp{\mathop{\rm exp}}

\def\EE{\Expect}
\DeclareMathOperator\sign{\rm sign}

\def\Var{\mathrm{Var}}

\def\PP{\mathbb{P}}

\def\eqdef{\triangleq}

\def\upto{\nearrow}

\def\vect#1{\mathbf{#1}}

\def\Leb{\mathrm{Leb}}
\def\simiid{\stackrel{iid}{\sim}}

\newcommand{\Unif}{\mathrm{Uniform}}

\newcommand{\Expect}{\mathbb{E}}

\newcommand{\Ber}{\text{Ber}}

\definecolor{myblue}{rgb}{.8, .8, 1}
\definecolor{mathblue}{rgb}{0.2472, 0.24, 0.6} 
\definecolor{mathred}{rgb}{0.6, 0.24, 0.442893}
\definecolor{mathyellow}{rgb}{0.6, 0.547014, 0.24}

\newcommand{\calN}{{\mathcal{N}}}

\newcommand{\calP}{{\mathcal{P}}}

\newcommand{\calY}{{\mathcal{Y}}}

\def\unifto{\mathop{{\mskip 3mu plus 2mu minus 1mu%
	\setbox0=\hbox{$\mathchar"3221$}%
	\raise.6ex\copy0\kern-\wd0%
	\lower0.5ex\hbox{$\mathchar"3221$}}\mskip 3mu plus 2mu minus 1mu}}

\ifx\lesssim\undefined
\def\simleq{{{\mskip 3mu plus 2mu minus 1mu%
	\setbox0=\hbox{$\mathchar"013C$}%
	\raise.2ex\copy0\kern-\wd0%
	\lower0.9ex\hbox{$\mathchar"0218$}}\mskip 3mu plus 2mu minus 1mu}}
\else
\def\simleq{\lesssim}
\fi

\ifx\gtrsim\undefined
\def\simgeq{{{\mskip 3mu plus 2mu minus 1mu%
	\setbox0=\hbox{$\mathchar"013E$}%
	\raise.2ex\copy0\kern-\wd0%
	\lower0.9ex\hbox{$\mathchar"0218$}}\mskip 3mu plus 2mu minus 1mu}}
\else
\def\simgeq{\gtrsim}
\fi

\def\dperp{\perp\!\!\!\perp}

\newif\ifmapx
{\catcode`/=0 \catcode`\\=12/gdef/mkillslash\#1{#1}}
\edef\jobnametmp{\expandafter\string\csname main2_apx\endcsname}
\edef\jobnameapx{\expandafter\mkillslash\jobnametmp}
\edef\jobnameexpand{\jobname}
\ifx\jobnameexpand\jobnameapx
\mapxtrue
\else
\mapxfalse
\fi

\long\def\apxonly#1{\ifmapx{\color{blue}#1}\fi}

\newcommand{\co}{\mathop{\mathrm{co}}}

\renewcommand{\hat}{\widehat}
\renewcommand{\tilde}{\widetilde}

\ifarxiv
\author{Meir Feder (Tel-Aviv University) and Yury Polyanskiy (MIT)}
\else
\coltauthor{%
 \Name{Meir Feder} \Email{meir@tauex.tau.ac.il}\\
 \addr School of Electrical Engineering, Tel-Aviv University, Israel
 \AND
 \Name{Yury Polyanskiy} \Email{yp@mit.edu}\\
 \addr Dept. of EECS, MIT, USA
}
\fi

\begin{document}

\maketitle 
\begin{abstract}%
We consider the question of sequential prediction under the log-loss in terms of cumulative regret. Namely, 
given a hypothesis class of distributions, learner sequentially predicts the (distribution of the) next letter in sequence
and its performance is compared to the baseline of the best constant predictor from the hypothesis class. The
well-specified case corresponds to an additional assumption that the data-generating distribution belongs to the
hypothesis class as well. Here we present results in the more general misspecified case. Due to special properties of the
log-loss, the same problem arises in the context of competitive-optimality in density estimation and model selection. 
For the $d$-dimensional Gaussian location hypothesis class, we show that cumulative regrets
in the well-specified and misspecified cases asymptotically coincide. In other words, we
provide an $o(1)$ characterization of the distribution-free (or PAC) regret in this case -- the first such result 
as far as we know. We recall that the
worst-case (or individual-sequence) regret in this case is larger by an additive constant ${d\over 2} + o(1)$.
Surprisingly, neither the traditional Bayesian estimators, nor the Shtarkov's normalized maximum
likelihood achieve the PAC regret and our estimator requires special ``robustification'' against heavy-tailed data. 
In addition, we show two general results for misspecified regret: the existence and uniqueness of the optimal
estimator, and the bound sandwiching the misspecified regret between well-specified regrets with (asymptotically)
close hypotheses classes.
\end{abstract}

\ifarxiv
\else
\fi



\ifarxiv
\tableofcontents
\fi

\section{Introduction}

This paper considers problems of the following type:
$$ \min_Q \max_{P\in \Phi,P^* \in \Theta} \EE_{Y\sim P}\left[ \log {P^*(Y)\over Q(Y)}\right]\,,$$
where $\Theta$ and $\Phi$ are some collections of distributions. The goal is to find the (approximate) value of $\min\max$ and the
 (approximate) minimizer $Q^*$. There are several ways in which this abstract problem can
arise (see Section~\ref{sec:motiv}). The problem has been studied in information theory, statistics and machine learning
predominantly in the following two cases: when $\Phi=\Theta$ (well-specified or ``stochastic'' case) and  when $\Phi$
consists of all distributions (worst-case or ``individual-sequence''). However, the natural intermediate
case of when $\Phi$ consists of all iid distributions (a case we designate by the name ``PAC'') has not been studied as
much. We report new results pertaining to the cases of $\Phi \neq \Theta$. Such a
setting has been known under the names of model-mismatch, misspecified regret, agnostic learning or distribution-free
PAC (in case $\Phi$ is all iid distributions). Our paper can be filed under either of these. 

\paragraph{Notation.} We use $P\ll Q$ for absolute continuity of measures, $dP\over dQ$ for Radon-Nikodym derivatives, $D(P\|Q) = \EE_P[\log {dP\over dQ}]$ for Kullback-Leibler (KL) divergence, $I(X;Y) = D(P_{X,Y}\|P_X \otimes P_Y)$
for mutual information, $P \otimes Q$ for a product measure, $P^{\otimes n}$ for an $n$-fold product of $P$ with itself,
$y^n = (y_1,\ldots,y_n)$ for an $n$-vector, $\calY^n$ for a measurable space of $n$-vectors, $\calP(\calY)$ for a set of
probability measures on a measurable space $\calY$, $\calP_{iid}(\calY^n) = \{P^{\otimes n}: P \in \calY(Y)\}$, 
$[k] = \{1,\ldots,k\}$ for $k\in \mathbb{Z}_+$, $\Leb$ denotes Lebesgue measure.

\subsection{Defining regret under model misspecification}

Fix a measurable space $\calY$ and a collection of hypotheses $P_\theta, \theta \in \Theta$
of measures on it, which we will call model class $\Theta$. Suppose an iid sequence $Y_i\sim P$ is observed and our goal 
is to provide an estimate of its distribution that is (almost) as good as the best possible hypothesis
$P_{\theta^*}$. More specifically, suppose that 
having observed $y_1,\ldots,y_{t-1}$ we output our estimate distribution $Q_t(\cdot) = Q_{t}(\cdot| y^{t-1})$, 
and then upon observing $Y_{t}$ experience a (relative) regret of $\log {dP_{\theta^*}\over
dQ_{t}}(Y_{t})$. Our goal is to minimize
\begin{equation}\label{eq:i1}
 	\sup_{\theta^*, P} \EE \left[\sum_{t=1}^n \log {dP_{\theta^*}\over
dQ_{t}}(Y_{t})\right]\,,
\end{equation}
where supremum over $\theta^*$ corresponds to chosing the best in-model match and supremum over $P$ corresponds to
the worst-case choice of the data generating distribution. (Non-iid models and/or generating distributions, 
e.g. Markov processes, can be handled by taking $n=1$ and extending $\calY$.)

The most studied case of this problem is the \textit{well-specified} case, when in addition we restrict supremum over $P$ to $P=P_\theta$ for some $\theta \in \Theta$. In this case, it is clear that the 
optimal choice of $\theta^*=\theta$ and we get the well-known definition of the the optimal minimax (cumulative) regret, called the capacity of $\Theta$:
$$ C_n(\{P_\theta, \theta \in \Theta\}) = C_n(\Theta) = \inf_{Q_{Y^n}}
\sup_{\theta \in \Theta} \EE^{\theta} \left[\log {dP_{\theta}^{\otimes n} \over
dQ_{Y^n}}(Y^n) \right] \,.$$

A simple observation shows that 
\begin{equation}\label{eq:cap_def}
	  C_n(\Theta) = \inf_{Q_{Y^n}} \sup_{\theta \in \Theta} D(P_{\theta}^{\otimes n} \| Q_{Y^n}) \,.
\end{equation}
A fundamental theorem of Kemperman~\cite{kemperman1974shannon} states that whenever $C_n(\Theta)<\infty$ there exists a unique $Q^*_{Y^n}$ such that
$$ C_n(\Theta) = \sup_{\theta \in \Theta} D(P_{\theta}^{\otimes n} \| Q^*_{Y^n}) \,, $$
and, furthermore,
\begin{equation}\label{eq:capred}
	C_n(\Theta) = \sup_{\pi} I(\theta; Y^n)\,,
\end{equation}
where supremum is over all (finitely supported) priors on $\Theta$. In application to sequential prediction, this result
is also known as the capacity-redundancy
theorem~\cite{gallager1979source,ryabko1979coding,davisson1980source} and its strong version is given in \cite{merhav1995strong}.

Notice that, in particular, whenever $C_n(\Theta) < \infty$ there must exist a measure $\mu$ such that $P_{\theta} \ll \mu$ for all $\theta$ (e.g. one can take $\mu = Q^*_1$). Thus, in the sequel we fix an auxiliary measure $\mu$ on $\calY$ and assume that 
$$ P_\theta(dy) = f_\theta(y) \mu(dy)\,,$$
that is the family $P_\theta$ is given by its relative densities $f_\theta$.

In this paper we study the misspecified case where the supremum over data-generating distributions does not have to come from the model class $\Theta$ (and in fact is not even required to be iid). 
\begin{definition} For a given $\calY$, $\mu$, a collection of densities $\{f_\theta,\theta \in\Theta\}$ and a collection of distributions $\Phi_n$ on $\calY^n$ we define
	\begin{equation}\label{eq:f_def}
	 F_n(\{f_\theta,\theta\in\Theta\}, \mu, \Phi_n) \eqdef \inf_q \sup_{P \in \Phi_n} \sup_{\theta \in \Theta} 
	\EE_{Y^n\sim P}\left[\log {\prod_{t=1}^n f_\theta(Y_t)\over q(Y^n)}\right]\,,
\end{equation}
where infimum is over all $q: \calY^n \to \mreals_+$ with $\int q d\mu^{\otimes n} = 1$. When $n=1$ and $\Phi \subset
\calP(\calY)$, we 
shorten $F_1(\{f_\theta,\theta\in\Theta\}, \mu, \Phi)$ to just $F(\Theta,\Phi)$. 
\end{definition}
There are three subtleties (discussed in detail in Appendix~\ref{app:tech_rem}): zeros under $\log$, non-existence of
$\EE$ and the fact that quantity $F_n$ \textit{may} depend on a choice of densities $f_\theta$ for representing $\{P_\theta\}$. 

In the most extreme case, we take $\Phi_n=\calP(\calY^n) \eqdef \{\mbox{all distributions on~}\calY^n\}$. The resulting
quantity is known as the individual-sequence regret:
\begin{equation}\label{eq:g_def}
	\Gamma_n(\{f_\theta,\theta \in \Theta\}, \mu) = \inf_{q} \sup_{y^n \in \calY^n} \sup_{\theta \in \Theta} \log {\prod_{t=1}^n f_\theta(y_t)\over q(y^n)}\,, 
\end{equation}
A result of Shtarkov~\cite{Shtarkov88} shows that that infimum in the definition is achieved by 
\begin{equation}\label{eq:shtarkov}
	 q(y^n) = e^{-\Gamma_n} \bar f(y^n), \quad \bar f(y^n) \eqdef \sup_\theta \prod_{t=1}^n f_\theta(y_t)\,,
\end{equation}
assuming that (a) $\bar f(y^n)$ is measurable; and (b) that $\int \bar f d\mu^{\otimes n} = e^{\Gamma_n} < \infty$.

From the learning point of view, the most interesting case is perhaps 
$\Phi_n = \calP_{iid}(\calY^n) \eqdef \{\mbox{all iid distributions}\}$, which
corresponds to the fully distribution-free regret (or agnostic learning). We denote this special case by $F_n^{(PAC)}(\{f_\theta\},\mu)$. Note that we always have
$$ C_n \le F_n^{(PAC)} \le \Gamma_n\,.$$
The main motivation for this work was to understand whether $F_n^{(PAC)}$ is closer to $C_n$ or $\Gamma_n$. All of the
results in this paper suggest the former, thus providing certain justification for the classical focus on the
well-specified case.

Our first such result is the following.
\begin{theorem}\label{th:glm} Let $\calY = \mreals^d$, $\mu = \Leb$ and $f_\theta(y) = (2\pi)^{-{d\over2}}
e^{-{1\over2}\|y-\theta\|^2}$, $\Theta$ -- a compact subset of $\mreals^d$ with $\Leb(\Theta)>0$. Then we have
	$$ F_n^{(PAC)} = C_n(\Theta) + o(1)\,,$$
	whereas $\Gamma_n = C_n(\Theta) + {d\over 2} + o(1)$. The estimator we construct simultaneously
	achieves $C_n(\Theta)+o(1)$ in the PAC setting and $\Gamma_n + o(1)$ in the individual-sequence setting.
\end{theorem}
There are several surprises about this result (see Section~\ref{sec:glm_discussion} for details). 
First, the Shtarkov distribution~\eqref{eq:shtarkov} only achieves a suboptimal regret of
$\Gamma_n$. Thus, this means that there exist an online predictor which is able
to exploit the special structure of the iid generated data and therefore reduce
the regret compared to the fully adversarial case of $\Gamma_n$. Second, the
distribution $Q_{Y^n}$ that is (asymptotically) optimal for $C_n$, namely, the
Bayes average over the Jeffrey's prior\footnote{Explicitly, $Q_{Y^n}(\cdot) =
\int_{\Theta} \Leb(d\theta) P_{\theta}^{\otimes n}(\cdot)$.} similarly does not
achieve $F_n^{(PAC)}$. Third, our proof strongly suggests that the optimal predictor should provide some robustification 
against the cases when $Y$ is heavy tailed.

\smallskip

Our next result concerns collections $\Phi_n$ slightly smaller than $\calP_{iid}(\calY^n)$. We only need to state the $n=1$ result.
\begin{theorem}\label{th:exist} Suppose that $\Phi\subset \calP(\calY)$ is such that for every $P\in \Phi$ we have
\begin{equation}\label{eq:dist_th}
		D(P\|\Theta) \eqdef \inf_{\theta \in \Theta} D(P\|P_\theta) < \infty\,.
\end{equation}	
Suppose $F(\{P_\theta\},\mu, \Phi)<\infty$, then there exists a unique distribution $Q^* \ll \mu$ with density $q^*$ 
such that 
	$$ 
	F(\{f_\theta\},\mu, \Phi) = \sup_{P\in \Phi} \sup_{\theta\in\Theta} \EE_{Y\sim P}
	\left[\log {f_{\theta}(Y) \over q^*(Y)}\right]\,.$$
Furthermore, for every $P\in \Phi$ we have $D(P\|Q^*) \le F + D(P\|\Theta) < \infty$.
\end{theorem}

To appreciate weakness of the condition in the preceding theorem, notice that $D(P\|\Theta)=\infty$ means that average
loss of the oracle estimator is infinite, since $\inf_\theta D(P\|P_\theta)=\infty$. However, as Theorem~\ref{th:glm}
shows even when both losses are infinite in expectation, their difference may still be bounded (just notice that
$D(P\|\Theta)=\infty \Leftarrow \EE_{P} [\|X\|^2] = \infty$ and thus there are plenty of such $P$).

\smallskip
Our final result is about a further smaller collections $\Phi$. 
\begin{theorem}\label{th:small} Suppose that $\Phi \subset \calP(\calY)$ is such that a) $P\ll \mu$ for every $P\in
\Phi$; b) $C_n(\Phi) = \tau_n n$, $\tau_n\to 0$. Then for every $\epsilon_n \gg \tau_n$ we have 
	$$ F_n(\Theta, \Phi^{\otimes n}) \le F_n(\Theta, \Theta_{\epsilon_n}^{\otimes n}) + o(1)\,,$$
where $\Theta_\epsilon = \{P \in \Phi: D(P\|\Theta) \le \epsilon\}$, $\Phi^{\otimes n} = \{P^{\otimes n}: P \in \Phi\}$
and similarly for $\Theta_{\epsilon}^{\otimes n}$.
\end{theorem}

The meaning of this last result is the following. Since $F_n(\Theta, \Theta_{\epsilon}^{\otimes n}) \le C_n(\Theta_\epsilon)$, we
conclude that for any $\epsilon>0$ we have
	\begin{equation}\label{eq:small_disc}
		C_n(\Theta) \le F_n(\Theta, \Phi^{\otimes n}) \le C_n(\Theta_\epsilon) + o(1)\,.
\end{equation}	
Since $\Theta_\epsilon$ is a very small enlargement of the model class $\Theta$, in many cases we will have
$C_n(\Theta_\epsilon) \to C_n(\Theta)$ as $\epsilon\to 0$ (but not always -- see example in Section~\ref{sec:cex_cap}). 
In such cases, taking $\epsilon\to 0$ sufficiently slowly
will recover $F_n(\Theta,\Phi^{\otimes n}) = C_n(\Theta) + o(1)$ -- the same result we have shown in
Theorem~\ref{th:glm} but for a much larger misspecfication ($\Phi$ used there certainly has $C_n(\Phi)=\infty$). 

The practical implication of Theorem~\ref{th:small}, thus, is that the optimal misspecified regret equals to, and (almost)
optimal estimators can be constructed by reducing to the well-specified case with a slightly enlarged model class
$\Theta_\epsilon$ (to which one is free to apply Shtarkov or Jeffreys prior estimators). This last message can be
demonstrated heuristically via the following chain (we ignore all rigor and appeal to the intuition here). Define $c(P)
			\eqdef \inf_\theta D(P\|P_\theta)$ and consider
\begin{align} \lefteqn{F_n(\Theta, \Phi^{\otimes n})} \nonumber\\
&= \inf_{Q_{Y^n}} \sup_{P, \theta} \EE_{Y^n\simiid P} \left[\log {P_\theta^{\otimes
n}(Y^n)\over Q_{Y^n}(Y^n)}\right] = \inf_{Q_{Y^n}} \sup_{P, \theta} \EE_{Y^n\simiid P} \left[\log {P_\theta^{\otimes n}(Y^n)\over
			Q_{Y^n}(Y^n)} {P^{\otimes n}(Y^n)\over P^{\otimes n}(Y^n)} \right]\nonumber\\
		 	&= \inf_{Q_{Y^n}} \sup_{P} D(P\|Q_{Y^n}) - n \inf_\theta D(P\|P_\theta) 
			= \inf_{Q_{Y^n}} \sup_{P} D(P\|Q_{Y^n}) - n c(P) \nonumber\\
		 	&= \inf_{Q_{Y^n}} \sup_{\pi} \EE_{\phi \sim \pi}  \left[D(P_\phi\|Q_{Y^n}) - n c(P_\phi)\right]
			= \sup_{\pi} \inf_{Q_{Y^n}} \EE_{\phi \sim \pi}  \left[D(P_\phi\|Q_{Y^n}) - n c(P_\phi)\right]\nonumber\\
			&= \sup_{\pi} I(\phi; Y^n) - n \EE_\pi[c(P)]\,,\label{eq:fn_mi}
\end{align}
where for clarity we introduced index $\phi$ indexing over all distribution $P_\phi \in \Phi$, and a prior $\pi$ on the
random variable $\phi$ on $\Phi$ and used the fact that min-max equal max-min for convex-concave functions. Suppose,
furthermore, that $F_n(\Theta, \Phi^{\otimes n}) \le \Gamma_n(\Theta) = o(n)$. Then, this implies that the least-favorable prior
$\pi$ in the final equation above must satisfy
	$ \EE_\pi[c(P)] \le {1\over n} \Gamma_n(\Theta) \to 0$.
This means that while the misspecified setting permits $Y^n \simiid P$ for any $P\in \Phi$, in
reality restricting  $P$ to those
with $D(P\|\Theta)= o(1)$ results in a vanishing influence on regret. This last statement can be taken as a summary of our work.

\paragraph{Structure of the paper.} In the next Section~\ref{sec:motiv} we present motivation for studying
$F(\Theta,\Phi)$. Section~\ref{sec:glm} proves Theorem~\ref{th:glm} and discusses the surprises mentioned above.
Appendices are devoted to the proofs of the two other Theorems. Finally, Appendix~\ref{app:open} lists some of the open
questions we consider interesting.

\subsection{Motivation and literature}\label{sec:motiv}

\paragraph{Motivation.} Why would one consider quantity like $F(\Theta,\Phi)$? The distinguishing property of the log-loss
is that the same quantity appears simultaneously in three conceptually very different tasks: sequential prediction,
online density estimation and model selection.

First, consider the \textit{sequential prediction}, we think of $\Theta$ as hypothesis class, and the learner's goal is to
predict $Y^n \simiid P$ as good as the best hypothesis in the class, which is given by $P_{\theta^*} = \arg\min_{\theta
\in \Theta} \EE_{P} \log \frac{1}{P_\theta}$. The cumulative regret of the learner $Q_t(\cdot|Y^{t-1})$ with respect to
the hypothesis class $\Theta$ becomes
\begin{equation}\label{eq:mot_1}
	\mathrm{Reg}(\{Q_t\}, \Theta, \Phi) \eqdef \sup_{P\in \Phi}  \sum_{t=1}^n  \EE_{Y^n \simiid P} [\log {1\over Q_t(Y_t|Y^{t-1})} -  \log
{1\over P_{\theta^*}(Y_t)} ]\,.
\end{equation}
And, clearly, the problem $F_n(\Theta, \Phi^{\otimes n})$ corresponds to searching for a learner that minimizes this
regret. See~\cite[Chapter 9]{Nicolo} and~\cite{merhav1998universal} for more.

Second, let us replace both numerators in~\eqref{eq:mot_1} with $P(Y_t)$ to get 
\begin{equation}\label{eq:mot_2}
	\mathrm{Reg}(\{Q_t\}, \Theta, \Phi) = \sup_{P\in \Phi}  \sum_{t=1}^n  \left\{\EE_{Y^n \simiid P} [D(P\|Q_t)] - \min_{\theta\in \Theta}
	D(P\|P_\theta)\right\}
\end{equation}
We see that now regret of the learner $Q_t$ corresponds to a problem of \textit{density estimation}. Indeed, consider
first the case of $\Phi = \Theta$, in which case the last term is zero and the regret becomes simply the cumulative
KL-divergence loss. Thus, the problem $F_n(\Theta, \Phi^{\otimes n}) = C_n(\Theta)$ is merely a cumulative version
of the (improper) density estimation of the class $\Theta$. This observation leads to sharp results in statistics, as
pioneered by~\cite{haussler1997mutual,yang1999information}. 

The misspecified case of $\Phi \neq \Theta$ corresponds, then, to the competitive optimality variation of the density
estimation. Importance of this problem was highlighted by the influential~\cite{orlitsky2015competitive}, who considered
estimating large-alphabet discrete distributions. They noted that
estimators achieving $\min_{\{Q\}} \max_{P\in \Phi} \EE[D(P\|Q)]$ are empirically rather uninteresting. However, by selecting a
natural class $\Theta$ and seeking to minimize~\eqref{eq:mot_2} one does discover interesting estimators.

Third, in \textit{model selection}, one seeks to compare two models $\Theta_1$ vs $\Theta_2$ given
observations $Y^n \simiid P$. In the spirit of maximum likelihood (or minimal KL-divergence), a natural way to
decide which model fits the data better would be to compare
$$ \sup_{\theta \in \Theta_1} \EE_{Y\sim P}[\log P_\theta(Y)] \lessgtr \sup_{\theta \in \Theta_2}
\EE_{Y\sim P}[\log P_\theta(Y)]\,.$$
However, this requires computing population averages w.r.t. $P$. Attempt to fix this issue by 
replacing $\EE_{Y\sim P}[\log P_\theta(Y)]$ with ${1\over n}\sum_{i=1}^n \log P_{\theta}(y_i)$
results in well-known significant biases for large models $\Theta$. The idea
behind the minimum description length (MDL)
principle~\cite{grunwald2007minimum} is to associate with each model $\Theta_i$  a certain
``composite likelihood'' $\log Q_i(y^n)$, where
each $Q_i$ is chosen to satisfy for all $P\in \Phi$
and $Y^n \simiid P$ 
$$ {1\over n} \EE[\log Q_{i}(Y^n)] \approx \sup_{\theta \in \Theta_i} \EE[\log P_{\theta}(Y)]\,.$$
Clearly, the $Q_i$ that makes this $\approx$ the tightest is the one to minimize $\mathrm{Reg}(\{Q_t\},
\Theta_i,\Phi)$. We note that in the classical incarnation of the MDL, one uses either $\Phi =
\Theta_i$ or $\Phi=\{\text{all distributions}\}$ for defining $Q_i$. Following the results in this
paper, we propose that taking $\Phi$ to be all i.i.d. distributions would result in model
selection criteria much more robust to outliers and deviations. 

In all, we suggest that a robustified MDL should be implemented as follows: To compare two model
classes $\Theta_1$ and $\Theta_2$, one (a) finds good predictors (closely) attaining
$F_n^{(PAC)}(\Theta_i)$ for $i=\{1,2\}$; (b) runs each predictor against the sequence
$y_1,\ldots, y_n$; and (c) the winning model is the one whose predictor incurred smaller log-loss.

\apxonly{So the full YP-MDL idea is this:
\begin{itemize}
\item Set $\Theta=\Theta_1$, $\Phi$=all iid. Find $Q_1$.
\item Set $\Theta=\Theta_2$, $\Phi$=all iid. Find $Q_2$.
\item Select hypothesis depending on 
	$$ \log Q_1(y^n) >< \log Q_2(y^n) $$
\item Note that this is slighly more reasonable than the original MDL: afterall if you are
comparing $\Theta_1$ to $\Theta_2$ there is some inherent misspecification possible, so you do
need ``robustification''. Perhaps this could be used to show that Csiszar-Shields counter-example
(failure to estimate order of the MC via MDL) can be repaired.
\end{itemize}
}

\paragraph{Literature.} We do not survey the enormous literature available on well-specified $C_n$ and
individual-sequence $\Gamma_n$ problems and refer to~\cite{Nicolo,merhav1998universal}. 

For the misspecified case, the most directly relevant paper we are aware of is~\cite{takeuchi1998robustly}, which
focuses precisely on $F_n(\Theta,\Phi_n)$. However, the methods there are purely applicable to 
finite-dimensional parametric models and restricted classes $\Phi_n$, which certainly do not include the case of
$F_n^{(PAC)}$.%
\footnote{Note that~\cite{takeuchi1998robustly} only contains proof sketches, and furthermore different versions we 
found online state different conditions on $\Phi$. One of the authors of~\cite{takeuchi1998robustly} confirmed to us
that the more complete version is still in preparation.}

Another relevant work~\cite{grunwald2005asymptotic} studied
the plugin maximum-likelihood (ML) predictors $Q_t(\cdot|Y^{t-1}) =
P_{\hat\theta_{ML}(Y^{t-1})}$, where $\hat \theta_{ML}(Y^{t-1})$ is a (slightly modified) ML estimate of $\theta$. They
showed that in a 1-parameter exponential families the model misspecification leads to a
regret ${c\over 2} \log n + O(1)$, where $c={\Var_P[Y]\over \Var_{P_{\theta^*}}[Y]}$. 
So compared to the optimal scaling (of e.g. Shtarkov's predictor) plugin ML 
estimator can be suboptimal by an arbitrarily bad factor. Subsequently,~\cite{grunwald2010prequential} showed that
replacing $\hat \theta_{ML}$ above with any other function $\hat \theta(Y^{t-1})$ (different for all $t$), results in 
exactly the same scaling of regret. 

A notable alternative to a plug-in ML estimator is the sequential normalized ML (SNML)
estimator, which is simply a conditional Shtarkov
distribution~\cite{rissanen2007conditional}, and a close relative of the last-step-minimax algorithm
of~\cite{takimoto2000last}. 
Regarding these estimators, it was shown~\cite{kotlowski2011maximum} that the SNML does achieve the correct  ${d\over 2}
\log n + O(1)$ regret in the well-specified case. It turns out that SNML and the Bayes estimators with Jeffreys prior coincide whenever they are
optimal~\cite{bartlett2013horizon}.

\apxonly{

\textbf{Original:} Going away from the log-loss, under the absolute loss and finite $\calY$, it turns out that $\Gamma_n, F_n^{(PAC)}$ and
$C_n$ may all be equal. Indeed, for the binary $\calY=\{0,1\}$ Cover~\cite{cover1965behavior} shows that the cumulative
regret with respect to all constant predictors does not change if the worst-case loss is replaced with loss over iid
data $Y^n$ (in fact, $Y_i \simiid \Ber(1/2)$ achieves the worst-case loss).

\textbf{Meir:} For individual sequences you get $\sqrt(n)$ regret, but for the stochastic setting, the rolling-majority
decoder achieves constant regret for every $\theta$ (the constant explodes as $\theta \to 1/2$). See ``Some properties
of Sequential Predictors for Binary Markov Sources'' (Merhav-Feder-Gutman). 

\textbf{Resolution:} Ok, it turns out that from Cover's paper we get $\bar C_n=\Gamma_n$ (for absolute loss), but not
clear for $C_n$. Here we define $\bar C_n = \inf_Q \sup_P \EE[\mathrm{loss}(Q,Y^n) - \inf_\theta
\mathrm{loss}(f_\theta,Y^n)]$ and $C_n = \inf_Q \sup_P \EE[\mathrm{loss}(Q,Y^n)] - \inf_\theta\EE[
\mathrm{loss}(f_\theta,Y^n)]$.
}

The space of questions and amount of literature dramatically expands once we incorporate regressors $X_i$ into the
picture, so that prediction of $Y_{t+1}$ is done on the basis of $(X_1,\ldots,X_{t+1},Y_1,\ldots,Y_t)$. Here, the gap
between the PAC-misspecified and worst-case is very easy to demonstrate (take $Y_i = 1\{X_i \le \theta\}$ with $\theta
\in [0,1]$ -- the 1D-barrier -- which cannot be predicted, $\Gamma_n=\infty$, in the worst case, but is easy in the iid
case). For increasingly more general losses,~\cite{rakhlin2010online,rakhlin2014online,rakhlin2015online} show that
regret can be sharply characterized by the metric-entropy type quantites (sequential Rademacher complexity). However,
for the log-loss it turns out that the entropic characterization is not possible, cf.~\cite{bilodeau2020tight} building on
the predictor from~\cite{rakhlin2015sequential}.
Other recent results about non-parametric models under log-loss can be found
in~\cite{fogel2018universal,Grunwald2020_jmlr}, which study non-cumulative (batch) regret under the misspecification.

The SNML idea was extended to the case of regressors in~\cite{fogel2018universal} under the name of predictive NML
(pNML). 
Subsequently,~\cite{gastpar2020} demonstrated that the role of training is to focus attention to a
smaller subclass of $\Theta$, on which one may perform the NML.

\section{Proof of Theorem~\ref{th:glm}}\label{sec:glm}

Let us now restrict attention to the special case of a compact set $\Theta \subset \mreals^d$ and the following model
class (Gaussian Location Model, GLM):
	\begin{equation}\label{eq:m_0}
		p_{\theta}(y) = (2\pi)^{-d/2} e^{-\|y-\theta\|^2/2}\,.
\end{equation}	
Define also 
\begin{equation}\label{eq:in_def}
	I_n = {d\over 2} \ln {n\over 2\pi e} + \ln \Leb(\Theta)\,.
\end{equation}
Standard results, e.g.~\cite{XieBarron00}, show that
\begin{equation}\label{eq:glm_asymp}
		C_n = I_n + o(1), \quad \Gamma_n = I_n + {d\over 2} + o(1)\,.
\end{equation}	

\begin{proof}[Proof of Theorem~\ref{th:glm}] Since it is clear that $F_n \ge C_n$ we only need to prove an upper bound. 
We will, thus, prove that for any $\delta > 0$ there exists an $n_0 = n_0(\delta)$ such that for
all $n\ge n_0$ we have
	$$ F_n \le I_n + \delta\,.$$

For any set $A \subset \mreals^d$ we define $d(x,A) = \inf\{\|x-y\|: y\in A\}$. 
For each $\tau\ge0$ we define compact sets $ \Theta_\tau = \{x: d(x,\Theta) \le \tau\}$ and the projector on
$\Theta_\tau$ as
$$ c_\tau(y) \eqdef \argmin_{x \in \Theta_\tau} \|x-y\|\,,$$
with ties resolved in such a way that the resulting function is measureable (that this is possible follows from the 
measurable selection theorem (due to Kuratowski and Ryll-Nardzewski): indeed for any open $U$ the set $\{y: d(y, U\cap
\Theta_\tau)=d(y,\Theta_\tau)|\}$ is measurable).

Fix $v\in \mreals_+$ and $\tau>0$ and define the Shtarkov distribution $P_S$ on $Y^n$ with density
	\begin{equation}\label{eq:shtarkov_vt}
		p_S(y^n) = {1\over Z_{v,\tau}} \sup_{\theta \in \Theta_\tau} (2\pi v)^{-nd/2} e^{-{1\over2v}\sum_{t=1}^n \|y_t -
	\theta\|^2}\,,
\end{equation}	
where $Z_{v,\tau}$ is the normalization constant, which satisfies  $\ln Z_{v,\tau} \approx I_n + {d\over 2}$ (when $v\to 1,\tau \to 0$ and $n\to \infty$) as shown in Lemma~\ref{lem:zvt} (Appendix~\ref{app:logz}).
\begin{remark} An important point of our analysis is the following. The Shtarkov distribution (with $v=1$ and $\tau=0$) 
achieves both the $C_n$ (up to $o(1)$) and $\Gamma_n$ (exactly). However, even for $Y\simiid \calN(\theta, \sigma^2
I_d)$ with $\sigma^2 < 1$ it yields a suboptimal regret. The choice $v=1-n^{-1}$ fixes this problem and makes Shtarkov
optimal for a class of all $O(1)$-subgaussian $Y$. However, for heavy-tailed $Y$ Shtarkov remains suboptimal
(Section~\ref{sec:glm_discussion}) and we need to incorporate some robustification into the estimator. This was surprising to us.
\end{remark}

We will freely use the following change of coordinates on $y^n$. Let $V_1=\{y^n: y_1 = \ldots = y_n \in \mreals^d\}$ be a
$d$-dimensional subspace and $V_\perp$ its orthogonal complement. We denote by $y_\perp$ the orthogonal projection of
$y^n$ onto $V_\perp$. 
We then have the following convenient orthogonal decomposition:
	$$ y^n = y_{\perp} + \bar y_n \otimes \vect{1}\,, \qquad \bar y_n = {1\over n} \sum_{t=1}^n y_t $$
	where $\vect{1}$ is the all-1 vector in $\mreals^n$ and for $a\in \mreals^d$, $b\in \mreals^n$ the
	$(a\otimes b)_{m} = a_{m \mod d} b_{\lceil m/d\rceil}$ is the standard Kronecker product.
	Note also that for any function $f(y^n) = \psi(y_\perp, \bar y_n)$ we have
\begin{equation}\label{eq:m_1}
			\int_{\mreals^{nd}} f(y^n) \Leb(dy^n) = n^{d/2} \int_{\mreals^{(n-1)d}} \Leb(dx_1) \int_{\mreals^d}
		\Leb(dx_2) \psi(x_1, x_2)\,.  
\end{equation}		

It is easy to see that we have
	\begin{align} p_S(y^n) 
		    &= {1\over Z_{v,\tau}}(2\pi v)^{-{nd/2}} e^{-{1\over 2v}(\|y_\perp\|^2 + n\|\bar y_n - c_\tau(\bar
		    y_n)\|^2)} \label{eq:m_9}\\
		p_\theta(y^n) &= (2\pi)^{-nd/2} e^{-{1\over 2}(\|y_\perp\|^2 + n \|\bar y_n - \theta\|^2)}\,.\label{eq:m_10}
\end{align}

Next fix $\alpha>0$ and define distribution $P_E$ on $\mreals^{nd}$ with density\footnote{Instead of $\|\bar y_n\|_1$ we
could equally well use $\|\bar y_n\|$ but then normalization constant would be more complicated. This is the only reason
for using the 1-norm.}
	$$ p_E(y^n) = (2\pi)^{-(n-1)d/2} e^{-{1\over 2}\|y_\perp\|^2} \times \left(\alpha \over 2\sqrt{n}\right)^d
	e^{-\alpha \|\bar y_n\|_1}\,,$$
where $\|v\|_1 = \sum_{j=1}^d |v_j|$ for any $v\in \mreals^d$. Using~\eqref{eq:m_1} we can check that this is indeed a
valid probability density.

Finally, fix $\lambda > 0$, set $\bar\lambda = 1-\lambda$ and define the estimator's density:
\begin{equation}\label{eq:my_estimator}
		q_{Y^n}(y^n) = \bar \lambda p_S(y^n) + \lambda p_E(y^n)\,.
\end{equation}	
We complete the proof by showing the following statement:  For every $\delta > 0$ we can select sufficiently small
$\tau,\lambda>0$ and sufficiently large $n_0$ and $\beta>0$ so that for all $n\ge n_0$ by setting 
\begin{equation}\label{eq:m_6}
	v=v_n = 1-{1\over n} + {\beta\over n^2}  
\end{equation}
we have
\begin{equation}\label{eq:m_2}
		\sup_{\theta \in \Theta} \EE_{Y_t \simiid P_Y} \left[R(Y^n; \theta)\right] \le I_n + \delta\,, \qquad \forall P_Y
\end{equation}	
where $ R(y^n; \theta) = \ln {p_{\theta}(y^n)\over q(y^n)}$, which we will upper bound as
\begin{equation}\label{eq:m_bd}
	R(y^n; \theta) \le \min(R_1(y^n; \theta), R_2(y^n; \theta))\,,
\end{equation}
where $R_1(y^n; \theta) \eqdef \ln{p_\theta(y^n)\over \bar\lambda p_S(y^n)}$ and 
$R_2(y^n; \theta) \eqdef \ln{p_\theta(y^n)\over \lambda p_E(y^n)}$.


First, we show that without loss of generality we may assume that $\EE[\|Y\|^2] < \infty$. We have
$$ R_2(y^n; \theta) = \tilde b_n + \alpha \|\bar y_n\|_1 - {n\over 2}
\|\bar y_n - \theta\|^2\,, $$
where $\tilde b_n = {d\over 2} \ln {4n\over 2\pi \alpha^2} - \ln {\lambda}$. 
From~\eqref{eq:m_bd} and $\|\bar y_n\|_1 \le \sqrt{d} \|\bar y_n\|$ we get
$$ R(y^n; \theta) \le \tilde b_n + \alpha \|\theta\|_1 + \alpha \sqrt{d} \|\bar y_n - \theta\| - {n\over 2} \|\bar
Y_n - \theta\|^2\,.$$
Note that for any random vector $X$ with $\EE[\|X\|^2] = \infty$ and any constants $a,b>0$ we must have
$ \EE[a\|X\|-b\|X\|^2] = -\infty$. 
Thus, if $\EE[\|\bar Y_n\|^2] = \infty$ then the expectation in~\eqref{eq:m_2} equals $-\infty$ and there is nothing to
prove. Consequently, we assume $\EE[\|\bar Y_n\|^2] < \infty$, which  by Lemma~\ref{lem:finsum} (Appendix~\ref{app:logz}) implies 
$\EE[\|Y\|^2] < \infty$, as claimed. We denote $\mu \eqdef \EE[Y]$ and $V \eqdef \EE[\|Y-\mu\|^2]$.

For $Y$ with finite two moments we can decompose 
$$ \sup_\theta \EE[R(Y^n; \theta)] = \mathrm{const} + \sup_{\theta\in\Theta} \EE[ -{n\over2}\|\bar Y_n - \theta\|^2]\,.$$
and thus the supremum over $\theta$ is attained at $\theta = c_0(\mu)$. Fixing $\theta = c_0(\mu)$ we have
\begin{align} R_1(y^n;\theta) 
		       &= a_n + {1\over 2} \sum_t {1\over v} \|y_t - c_\tau(\bar y_n)\|^2 - \|y_t - c_0(\mu)\|^2\,,
		       	 \qquad a_n \eqdef {nd\over 2} \ln v + \ln {Z_{v,\tau}\over \bar\lambda} \nonumber\\
   R_2(y^n;\theta) &= a_n + b_n + \alpha \|\bar y_n\|_1 - {n\over 2} \|\bar y_n - c_0(\mu)\|^2\,,
   			\qquad  b_n \eqdef \tilde b_n - a_n \nonumber
\end{align}
We transform expression for $R_1$ using the following identities valid for arbitrary $x\in \mreals^d$:
\begin{align} 
 	\sum_t \|Y_t - x\|^2 &= n\|\bar Y_n - x\|^2 + n\hat V\,, & \hat V &\eqdef {1\over n} \sum_t \|Y_t-\bar
	Y_n\|^2\nonumber\\
   \EE\left[\|\bar Y_n - x\|^2 \right] &= {V\over n} + \|x-\mu\|^2, & \EE\left[\hat V\right] &= {n-1\over n}
   V\,,\nonumber
\end{align}
Applying these to $R_1$ we get
\begin{align} \min(R_1,R_2) &= a_n - {n\over 2} \|\bar Y_n - c_0(\mu)\|^2 + 
	\min \left( {n\over 2v} \|\bar Y_n - c_\tau(\bar Y_n)\|^2  + {n\over 2}(v^{-1}-1) \hat V, b_n + \alpha
	\|\bar Y_n\|_1\right)\nonumber\\
	&\le a_n - {n\over 2} \|\bar Y_n - c_0(\mu)\|^2 + {n\over2}(v^{-1}-1) \hat V + \tilde R_n\,, \label{eq:m_4}\\
	\tilde R_n &\eqdef \min \left( {n\over 2v} \|\bar Y_n - c_\tau(\bar Y_n)\|^2, b_n + \alpha
	\|\bar Y_n\|_1 \right) \nonumber
\end{align}	

We now consider separately $\mu \in \Theta_\tau$ and $\mu \not\in\Theta_\tau$. Suppose the former and consider 
$\|\bar y_n - \mu\|\le \tau$. Then denoting $x = c_0(\mu) + \bar y_n - \mu \in \Theta_\tau$, we have from the
definition of $c_\tau$ that 
\begin{equation}\label{eq:m_3}
	\|\bar y_n - c_\tau(\bar y_n)\| \le  \|\bar y_n - x\| = \|\mu - c_0(\mu)\|\,.
\end{equation}
Also observe that $\beta_2 = \sup_n b_n $ and $\beta_3 = \sup_{\mu \in \Theta_\tau} \|\mu\|$ are both finite. Thus, for some
$\beta_4>0$
\begin{equation}\label{eq:m_5}
	\EE[(b_n + \alpha \|\bar Y_n\|_1) 1\{\|\bar Y_n - \mu\|> \tau\}] \le \beta_4 {V\over 2n}\,.
\end{equation}
Indeed, notice that $\|\bar Y_n\|_1 \le \sqrt{d} \|\bar Y_n\|_2 \le \sqrt{d} \|\bar Y_n - \mu\|_2 + \sqrt{d} \beta_3$.
From Chebyshev we have then $ \PP[\|\bar Y_n - \mu\|\ge \tau] \le {V\over n\tau^2}$. On the other hand, 
$$ \EE[\|\bar Y_n - \mu\| 1\{\|\bar Y_n - \mu\|\ge \tau\} \le {1\over \tau} \EE[\|\bar Y_n - \mu\|^2] = {V\over
n\tau}\,.$$
Combining these two estimates yields~\eqref{eq:m_5}.

We now bound $\EE[\tilde R_n]$ by retaining the first term of the minimum if $\|\bar Y_n -\mu\|\le \tau$ (and
invoking~\eqref{eq:m_3}) and the second term otherwise. This results in a bound 
$$ \EE[\tilde R_n] \le \beta_4 {V\over n} + {n\over 2v} \|\mu - c_0(\mu)\|^2\,.$$
Plugging the latter into~\eqref{eq:m_4} we have shown that for all $\mu \in \Theta_\tau$
$$ \EE[\min(R_1(Y^n),R_2(Y^n))] \le a_n + {n\over 2} \|\mu-c_0(\mu)\|^2 (v^{-1}-1) + {1\over2} V \gamma_n\,, $$
where $ \gamma_n = (v^{-1}-1)(n-1) + {\beta_4\over n} -1$. Recall that due to~\eqref{eq:m_6} we have
\begin{equation}\label{eq:m_7}
	{1-v\over v}(n-1) = {n-1\over n} {1-\beta/n\over 1-{1\over n} + {\beta\over n^2}} \le {n-1\over n} {1-\beta/n\over
1-{1\over n}} = 1-{\beta\over n} \,.
\end{equation}
Consequently, we have $\gamma_n \le {\beta_4 - \beta\over n}$. By chosing $\beta = \beta_4$ we obtain
$$ \EE[\min(R_1(Y^n),R_2(Y^n))] \le a_n + {n\over 2} \|\mu-c_0(\mu)\|^2 (v^{-1}-1)\,.$$
Again applying~\eqref{eq:m_7} we further upper bound   ${n\over 2} (v^{-1}-1) \le 1$ for all sufficiently large $n$, and
finally obtain
$$ \EE[\min(R_1(Y^n),R_2(Y^n))] \le a_n + \|\mu-c_0(\mu)\|^2 \le a_n + \tau^2\,,$$
since $\mu \in \Theta_\tau$. The proof of~\eqref{eq:m_2} in this case is completed after noticing that 
$$ a_n = -{d\over 2} + \ln Z_{v,\tau} - \ln \bar \lambda + o(1)$$
and that by Lemma~\ref{lem:zvt} (Appendix~\ref{app:logz}), sufficiently small $\tau$ and $\lambda > 0$ yield
$ a_n + \tau^2 \le I_n + \delta $.

Next consider $\mu \not \in \Theta_\tau$. In this case
\begin{align} \EE[\|\bar Y_n\|_1]  &\le \sqrt{d} (\EE[\|\bar Y_n - \mu\|] + \|\mu - c_0(\mu)\| + \|c_0(\mu)\|) \\
			&\le \sqrt{d} (\sqrt{V\over n} + \|\mu - c_0(\mu)\| + \beta_3) 
\end{align}
Noticing that $\tilde b_n = {d\over 2} \ln n + O(1)$ and denoting $\tilde \alpha = \sqrt{d}\alpha$
we get for some large constant $\beta_5$: the bound
\begin{align*} 
	\EE[R_2(Y^n; \theta)] &\le {d\over 2} \ln n + \beta_5  - {V\over 2}  + \tilde \alpha \sqrt{V \over n} - {n\over 2}\|\mu - c_0(\mu)\|^2 +
	\tilde\alpha \|\mu - c_0(\mu)\|
\end{align*}	
Note that $\sup_{n,V>0}-V/2 + \tilde\alpha \sqrt{V/n} < \infty $. Hence, the third and fourth terms can be absorbed into
the second. For sufficiently large $n$ the function $-{n\over 2}t^2 + \tilde\alpha t$ is monotonically decreasing on
$t\in [\tau,\infty)$ and thus we have shown
$$ \EE[R_2(Y^n; \theta)] \le {d\over 2} \ln n + \beta_5 - {n\over 2}\tau^2 + \tilde \alpha \tau\,.$$
Clearly, for sufficiently large $n$ the right-hand side of the last inequality is $\le I_n$. This concludes the proof
of~\eqref{eq:m_2}.

	Finally, to show the claim about $q_{Y^n}$ also achieving $\Gamma_n + o(1)$ we have to only notice that
	optimal $\lambda$ in~\eqref{eq:my_estimator} tends to 0 as $n\to\infty$ and thus $\log {p_S(y^n)\over q_Y(y^n)}
	= O(\lambda) \to 0$, implying that density $q_{Y^n}$ also attains $\Gamma_n+o(1)$ regret in the individual
	sequence setting.
\end{proof}

\apxonly{Reviewer asked if the following distribution:
$$ p_S(y^n) = {1\over Z_{v,\tau} v^{d/2}}(2\pi )^{-{nd/2}} e^{-\|y_\perp\|^2/2  - {n\over 2v}\|\bar y_n - c_\tau(\bar
		    y_n)\|^2} $$
would also work. The answer is no: Consider $Y_i = c$ with $c\in \Theta$. It is clear this will
achieve $\log Z_{v,\tau}$ loss.
}

Interestingly, a similar technique can be used to show a certain curious robustness result for maximum entropy. See
Appendix~\ref{app:diff_ent} for details.

\subsection{Discussion}\label{sec:glm_discussion}

\paragraph{Suboptimality of Shtarkov's estimator.}
As we remarked above, the Shtarkov distribution (which simultaneously achieves $C_n+o(1)$ in the well-specified setting
and $\Gamma_n$ in the worst-case one),
surprisingly, is suboptimal for the distribution-free misspecified case. 

First, let us focus on the case of $d=1$ and $\Theta=[-b,b]$. By
dividing~\eqref{eq:m_10} and~\eqref{eq:m_9}, we can derive that Shtarkov's distribution ($v=1, \tau=0$) achieves:
	$$ \sup_{\theta} \EE[ \log {p_{\theta}(Y^n)\over p_S(Y^n)} ] = \log Z_{1,0} + {n\over 2} \EE[(\bar Y_n - c(\bar
	Y_n))^2 - (\bar Y_n - c(\EE[Y])^2]\,,$$
	where we defined $c(y) = c_0(y) = \sign(y){|y|\cap b}$ and assumed that $Y^n \simiid P$ with $\EE[|Y|]<\infty$.
	If we furthermore, assume $\EE[Y] \in [-b,b]$ the we get
\begin{equation}\label{eq:lr_1}
	\max_\theta \EE \log {p_\theta(Y^n)\over p_S(Y^n)} = \Gamma_n - {1\over 2} \Var[Y] + {n\over 2} \EE[(c(\bar Y_n) -
\bar Y_n)^2]\,.
\end{equation}
Notice that if $Y_j \simiid \mathcal{N}(\theta, 1)$ with $\theta \in (-b,b)$ then the third term is $o(1)$ and we
conclude that Shtarkov's distribution attains $C_n(\Theta)+o(1)$ in the well-specified case\footnote{To handle
$\theta=\pm b$, we need to consider Shtarkov for an infinitesimally enlarged domain $[-b-\tau,b+\tau]$. In the case of
mismatched variance, i.e. when $\Var[Y] < 1$, we need to set $v=1-{1\over n}$ instead of $v=1$.}. However, in the
misspecified case the situation is different. Consider, the following heavy-tailed distribution 
	\begin{equation}\label{eq:lr_x}
		Y = \begin{cases} 0, & \text{w.p.~} 1-{1\over 4 b^2 n^2}\\
			 \pm 2bn, & \text{w.p.~} {1\over 8 b^2 n^2}
		\end{cases} 
\end{equation}		
This has $\EE[Y]=0$ and $\Var[Y]=1$, but notice the following issue. We have that with probability $\approx
{\mathrm{const}\over n}$
out of $n$ iid samples exactly one is going to take the value $2bn$, achieving $\bar Y_n = 2b$. Thus, while $\EE[Y] = 0$
we have $\PP[\bar Y_n - c(\bar Y_n) \ge b] \ge {\mathrm{const}\over n}$ and, therefore,
	$ {n\over 2} \EE[(c(\bar Y_n) - \bar Y_n)^2] \ge \mathrm{const}$.
Consequently,
	$$ \sup_{P \in \Phi}\max_\theta \EE \log {p_\theta(Y^n)\over p_S(Y^n)} \ge C_n(\Theta) + \mathrm{const} +
	o(1)\,,$$
implying that Shtarkov distribution does not achieve the optimal value of $F_n^{(PAC)} = C_n + o(1)$.

\paragraph{Suboptimality of Bayes estimator based on Jeffreys prior.}
Next, we want to show that the Jeffreys prior based estimator also does not achieve $F_n^{(PAC)}$. 
We consider the following density 
\begin{equation}\label{eq:m_11}
	p_J(y^n) = {1\over 2b} \int_{-b}^b d\theta p_\theta(y^n) = {1\over 2b\sqrt{n}} (2\pi)^{-(n-1)/2} e^{-{1\over 2}\|y_\perp\|^2} q_2(\bar y_n)\,,
\end{equation}
where we applied~\eqref{eq:m_10} and denoted the single-variate function
$$ q_2(y) = \PP[|y+ G_n| < b], \qquad G_n \sim\mathcal{N}(0,1/n)\,.$$
Dividing~\eqref{eq:m_10} by~\eqref{eq:m_11} we obtain :
	$$ \sup_{\theta} \EE[ \log {p_{\theta}(Y^n)\over p_J(Y^n)} ] =  \log {2b\sqrt{n}\over \sqrt{2\pi}} - \inf_\theta 
	\EE[{n\over 2}(\bar Y_n - \theta)^2 - \log q_2(\bar Y_n)]\,,$$
	or assuming that $\EE[Y] \in [-b,b]$ and $\Var[Y]=1$ we get (cf.~\eqref{eq:in_def})
	\begin{equation}\label{eq:m_12}
		\sup_{\theta} \EE[ \log {p_{\theta}(Y^n)\over p_J(Y^n)} ] = I_n  - \EE[\log q_2(\bar Y_n)]\,.
\end{equation}	
	When $Y_j \simiid \mathcal{N}(\theta,1)$ (or any subgaussian distribution), we have $q_2(\bar Y_n) =
	1-e^{-\Omega(n)}$ with high probability and thus, the last term in~\eqref{eq:m_12} is $o(1)$ and we see that
	indeed Jeffreys prior estimator achieves $C_n + o(1)$ regret in the well-specified case.

	However, when $Y_j\simiid P$ with distribution~\eqref{eq:lr_x} the problem occurs. As we argued above, with
	probablity $\ge {\mathrm{const}\over n}$ we get $\bar Y_n = 2b$, which implies $-\log q_2(\bar Y_n) =
	\Omega(n)$, and in turn  $\EE[-\log q_2(\bar Y_n)] >c > 0$. Hence, the regret of $p_J$ is suboptimal as well.

\paragraph{The nature of the gap between $F_n^{(PAC)}$ and $\Gamma_n$.} As a final remark, we
discuss the meaning of $F_n^{(PAC)} < \Gamma_n$. First, it is easy to show that if we take
$\Phi_n=\{\mbox{exchangeable distributions on~}\calY^n\}$ in~\eqref{eq:f_def} then the resulting
regret $F_n = \Gamma_n$.  One may be tempted to infer from this that
the PAC-optimal estimator is able to somehow exploit the inherent structure of the iid data (even adversarially
generated) and improve prediction compared  to the worst-case / exchangeable data. 
We argue, however, that it is not the quality of the estimator that deteriorates upon
relaxing the iid assumption, but rather the quality of the oracle predictor increases.  

More exactly, notice that when supremum
in~\eqref{eq:f_def} is evaluated over all (or all exchangeable) distributions, then the maximum is
attained at the extremal point, corresponding to a single known sequence $y^n$ (or its
permutations) as evidenced by~\eqref{eq:g_def}. This, in turn, gives the oracle more freedom as it
can adapt to the realization $Y^n$ as opposed to just the
distribution of it. This extra freedom is what results in the ${d\over 2}$ increase in regret, not
an iid structure per se. Indeed, let us redefine the regret in a well-specified case as follows:
$$
	\bar C_n(\Theta) \eqdef \inf_Q \sup_{\theta} \EE_{Y^n \simiid P_{\theta}} \left[
	\sup_{\theta^*} \log {dP_{\theta^*}\over
dQ}(Y^n)\right]\,,
$$
corresponding to oracle-estimator $P_{\theta^*}^{\otimes n}$ that is chosen given  the knowledge of $Y^n$. 
In Appendix~\ref{app:barcn} we argue that in most cases (and certainly in the GLM model of Theorem~\ref{th:glm}) we have
\begin{equation}\label{eq:barcn}
	\bar C_n(\Theta) = \Gamma_n + o(1)\,.
\end{equation}
This demonstrates that even in the well-specified case if we give oracle the power to adapt to
realization we do get the same regret as $\Gamma_n$, thus clarifying the nature of the gap between
$F_n^{(PAC)}$ and $\Gamma_n$.

\def\acktext{
We would like to thank Olivier Catoni for useful comments, Peter Gr\"unwald for
pointing out~\cite{takeuchi1998robustly}, and Alexander Rakhlin for useful discussions. We thank
anonymous reviewers for observations regarding exchangeable distributions in
Section~\ref{sec:glm_discussion}. The work of MF was partially supported by a grant from the
Israeli Science Foundation (ISF) Grant number 819/20.
The work of YP was supported in part
by the Center for Science of Information (CSoI),
an NSF Science and Technology Center, under grant agreement CCF-09-39370, NSF Grant ECCS-1808692, and the MIT-IBM Watson AI Lab.
}

\ifarxiv
\section*{Acknowledgement} 
\acktext
\bibliographystyle{plain}

\else
\acks{\acktext}
\fi

\appendix

\section{Technical remarks on the definition of $F(\Theta,\Phi)$}\label{app:tech_rem}

Unlike the well-specified case, where definition~\eqref{eq:cap_def} is elegant and rigorous, the regret in the
misspecified case, the $F_n$ and $\Gamma_n$, is more subtle. We list some of the issues in this section.

In~\eqref{eq:f_def} and below we adopt the following rules for evaluting $\log$ of the ratio: For any $c>0$ we set
\begin{equation}\label{eq:log_inf}
	\log {0\over c} = -\infty, \quad \log {c\over 0} = +\infty, \quad \log {0\over 0} = 0\,.
\end{equation}

In the infimum over densities $q$ in~\eqref{eq:f_def} we only consider \textit{admissible} $q$, i.e. those $q$ 
such that for every $P\in \Phi$ and every $\theta \in \Theta$ the expectation of the
log-likelihood ratio is well-defined (but possibly infinite)\footnote{Recall
that in Lebesgue integration theory we set $\EE[X] =
\EE[\min(X,0)]+\EE[\max(X,0)]$ unless the two summands are $-\infty$ and
$+\infty$, in which case the expectation is undefined.} If no such $q$ exist
then we take $F=\infty$.

A crucial subtlety concerning quantities $F_n$ and $\Gamma_n$ is that they are
not necessarily functions of the distributions $\{P_\theta\}$.  Rather they
depend on the particular chosen representative densities $f_\theta$. 
In some cases, corresponding to the same family $P_\theta$ one can
choose densities $f_\theta$ such that $\Gamma_n$ is increased by an arbitrary
amount. (One example is to take $P_\theta = \Unif[\theta, 1+\theta]$ and
compare densities $f_\theta(x)= 1\{\theta \le x \le 1+\theta\}$ with
$f_\theta(x) = 1\{\theta < x \le 1+\theta\} + 1\{x=\theta\}$.) For this reason,
we write $C_n(\{P_\theta\})$ but $\Gamma_n(\{f_\theta\}, \mu)$.\footnote{This suggests that perhaps a more 
sensible definition of $\Gamma_n$ would be to replace $\sup_\theta f_\theta(y^n)$ with the
definition of supremum common in the theory of Banach lattices~\cite[Section 2.6]{meyer2012banach}, namely
$\textrm{bsup}_\theta f_\theta(y^n)$ is defined as any function $g$ such that $f_\theta \le g$ ($\mu$-a.e. for all
$\theta$) and if $f_\theta \le h$ ($\mu$-a.e. for all $\theta$), 
then $g\le h$ ($\mu$-a.e.).  The advantages of $\mathrm{bsup}$ are that a) $\Gamma_n$ becomes insensitive to 
$\mu$-negligible modifications of $f_\theta$'s, and b)
$\mathrm{bsup}_\theta f_\theta(y^n)$ is automatically measurable if $\mu$ is $\sigma$-finite, cf~\cite[Lemma 2.6.1]{meyer2012banach}.}

Similarly, whereas for the well-specified case we can think of predictor as submitting at each step 
a distribution $Q_t(\cdot | Y^{t-1})$, in the misspecified case we insist that the predictor submits a density $q_t(\cdot|y^{t-1})$. 

Finally, note that none of these difficulties apply to the case when $\calY$ is countable.

\section{Technical results for Section~\ref{sec:glm}}\label{app:logz}

\begin{lemma}\label{lem:zvt} For any compact $\Theta$ with $\Leb(\Theta)>0$ we have 
	$$ \ln Z_{v,\tau} = {d\over 2} \ln {n\over 2\pi v} + \ln \Leb(\Theta_\tau) + o(1)\,,$$
	where $Z_{v,\tau}$ is the normalization constant from~\eqref{eq:shtarkov_vt}.
	Furthermore, as $\tau\to 0$ we have $\Leb(\Theta_\tau) \to \Leb(\Theta)$.
\end{lemma}
\begin{proof}
Let us define a function $d(z, B) = \inf_{z'\in B} \|z-z'\|$ which is continuous for any set $B$. We have
$$ \Leb(\Theta_\tau) = \int_{\mreals^d} 1\{d(z, \Theta) \le \tau\} dz\,.$$
Notice that as $\tau\to 0$ the sequence of functions $1\{d(z, \Theta) \le \tau\}$ converges pointwise to
$1\{z\in\Theta\}$. By compactness of $\Theta$, the set $\Theta_\tau$ is bounded and hence we have from the dominated
convergence $\Leb(\Theta_\tau) \to \Leb(\Theta)$ as $\tau \to 0$. This proves the second assertion.

To prove the first assertion we apply identity~\eqref{eq:m_1} to~\eqref{eq:m_9} to get
\begin{equation}\label{eq:glm_shtar}
	Z_{v,\tau} = \left({2\pi v\over n}\right)^{-d/2} \int_{\mreals^d} e^{-\alpha d(z, \Theta_{\tau})^2} dz\,, \quad
\alpha \eqdef {n\over 2v}\,.
\end{equation}
From here the result follows since for any bounded set $B$ we have
$$ \int e^{-\alpha d(z,B)^2} dz \to \Leb(B),\qquad \alpha \to \infty\,.$$
Indeed, since $B$ is bounded we can include it into the ball of radius $r$ and thus $d(z,B) \ge \left(\|z\|-r\right)_+$
for some finite $r>0$. Since $e^{-\left(\|z\|-r\right)_+^2}$ is integrable over $\Leb$ in $\mreals^d$, it  dominates all
integrands for $\alpha \ge 1$ and we have from the
dominated convergence theorem
	$$ \lim_{\alpha \to \infty} \int e^{-\alpha d(z,B)^2} dz  = \int \lim_{\alpha \to \infty} e^{-\alpha d(z,B)^2}
	dz = \Leb(B)\,.$$
\end{proof}

\begin{lemma}\label{lem:finsum} Let $X_i$ be iid random variables and let $\bar X_n = {1\over n} \sum_{t=1}^n X_t$. Then
$\EE[\bar X_n^2] < \infty$ iff $\EE[X_1^2] < \infty$.
\end{lemma}
\begin{proof}
We only need to prove that $\EE[(\bar X_n)^2] < \infty$ implies $\EE[(X_1)^2] < \infty$. To that end, let $\Psi(t) = \EE[e^{itX_1}]$ be the characteristic
function of $X_1$. From the conditions we know that the complex-valued function $f(t) = \Psi(t/n)^n$ is twice
continuously differentiable for all $t\in \mreals$ and $f(0)=1$. In a small neighborhood of $z=1$ on
the complex plane there exists an analytic function $f_1(z)$ satisfying $f_1(z)^n = z$. We have then $\Psi(t) =
f_1(f(nt))$ in a small neighborhood of $t=0$. In particular, $\Psi(t)$ is differentiable at $0$ and thus $X_1$ has a finite
second moment.
\end{proof}

\section{Robustness of maximum differential entropy}\label{app:diff_ent}
 For any
random variable $X$ with pdf $f_X$ let us define $h(X) = -\EE[\log f_X(X)]$ to be its differential entropy. It is well
known that 
	$$ \sup \{h(X): X\in[-b,b]\} = \log (2b)\,,$$
showing that uniform distribution has maximal entropy among all distributions with a given support. The following shows
a certain kind of robustness result.

\begin{lemma} As $\epsilon \to 0+$ we have
	$$ \sup \{h(M+Z): M \in [-b,b], \EE[Z]=0, \Var[Z] \le \epsilon\} = \log (2b) + O(\epsilon^{1/3} \log {1\over
	\epsilon})\,.$$
	where in the supremization we do not require $M$ and $Z$ be independent, but do require the distribution of
	$M+Z$ to have density.
\end{lemma}
\begin{proof} Since for any density $q$ and $dQ = q d\Leb$ we have $D(P_{M+Z}\|Q)\ge 0$, or in other words
	\begin{equation}\label{eq:l1}
		h(M+Z) \le -\EE[\log q(M+Z)]\,,
\end{equation}	
	Similar to the proof of Theorem~\ref{th:glm} we make the following choice ($\tau,\lambda,\alpha>0$ are to be
	chosen later):
	$$ q(x) = \bar \lambda {1\over 2(b+\tau)} 1\{|x| \le b+\tau\} + {\lambda \alpha\over 2} e^{-\alpha|x|}\,.$$
	The analysis of the RHS of~\eqref{eq:l1} proceeds similar to the proof of Theorem~\ref{th:glm}: 
	we consider two cases separately: $|Z| \le \tau$ and $|Z|>\tau$ and bounding $\log {1\over q(M+Z)}$ accordingly
	on each. Indeed, we have
	$$ \EE[1\{|Z|\le \tau\} \log {1\over q(M+Z)}] \le \log {2(b+\tau)\over 1-\lambda}\,.$$
	For the other case, denote $p=\PP[|Z|>\tau]$ and notice that $\EE[|M+Z| 1\{|Z|>\tau\}] \le bp + {1\over
	\tau}\EE[Z^2] \le bp + {\epsilon \over \tau}$ yielding
	$$ \EE[1\{|Z|> \tau\} \log {1\over q(M+Z)}] \le p (\log {2\over \lambda \alpha} + \alpha b) +
	{\alpha\epsilon\over \tau} \,.$$
	Taking here $\alpha=1$, $\lambda = \tau = \epsilon^{1/3}$ we obtain the result.
\end{proof}

\section{Justification of~\eqref{eq:barcn}}\label{app:barcn}

Recall that we are interested in checking
$$ \bar C_n(\Theta) \eqdef \inf_Q \sup_{\theta} \EE_{Y^n \simiid P_{\theta}} \left[ \sup_{\theta^*} {dP_{\theta^*}\over
dQ}(Y^n)\right] \stackrel{?}{=} \Gamma_n + o(1)$$
under some regularity assumptions (say smooth finite-parameter families). 
To that end note that the inner optimization can be solved explicitly yielding the Shtarkov density $P_S = \sup_\theta
p_\theta(y^n) \exp{-\Gamma_n}$. With this in mind we obtain:
\begin{align*} \bar C_n(\Theta) &= 
	\Gamma_n + \inf_Q \sup_{\theta} \EE_{Y^n \simiid P_\theta} \left[ \log {p_S(y^n)\over q(Y^n)}\right]\\
	&= 
	\Gamma_n + \inf_Q \sup_{\theta} D(P_\theta^{\otimes ^n}\|Q_{Y^n}) - D(P_\theta^{\otimes n}\| P_S)\,.
\end{align*}	

Now extending the inner supremum to supremum over priors on $\theta$ and lower-bounding $\inf \sup \ge \sup
\inf$\footnote{In
fact, there is always equality $\inf \sup = \sup \inf$, argued same as for~\eqref{eq:capred}, but we do not need this
extension.}, we get 
$$ \bar C_n(\Theta) \ge \Gamma_n + \sup_{\pi} I(\theta; Y^n) - \EE_{\theta \sim \pi}[D(P_\theta\|P_S)]\,.$$
Notice also that $I(\theta; Y^n) - \EE_{\theta \sim \pi}[D(P_\theta\|P_S)] = -D(P_\pi \|P_S)$, where we denoted $P_\pi =
\int \pi(d\theta) P_\theta^{\otimes n}$. In other words, we have shown that
	$$ 0 \le \Gamma_n - \bar C_n(\Theta) \le \inf_\pi D(P_\pi \|P_S)\,,$$
i.e. the gap between the two corresponds to how well the Shtarkov distribution can be approximated by a Bayes mixture.
In the GLM case, a simple explicit computation taking $\pi$ to be uniform on $[-b,b]$ (so that $P_\pi = P_J$
in~\eqref{eq:m_11}) yields 
	$$ \Gamma_n - \bar C_n(\Theta) = O(1/\sqrt{n})\,.$$
(See also~\cite[Appendix to Chapter 8]{grunwald2007minimum}.)

\section{Proof of Theorem~\ref{th:exist}}\label{app:exist}

The idea of the proof is best illustrated by looking at the heuristic derivation~\eqref{eq:fn_mi}, which concluded with
\begin{equation}\label{eq:i2}
		F(\Theta,\Phi) = \sup_{\pi} I(P; Y) - \EE_{\pi}[c(P)]\,,
\end{equation}	
	where supremum is taken over all prior distributions $\pi$ on $\Phi$, and the joint distribution of $P,Y$ is given
	by
	$$ \PP[P = a, Y=b] = \pi(a) P(b)\,.$$
	The expression~\eqref{eq:i2} is just a (Langrangian version of the ) cost-constrained channel capacity
	calculation.\footnote{Incidentally, this point of view also suggests a non-trivial method for finding the 
	optimal $Q^*$: the Blahut-Arimoto algorithm.} 
	Thus, a theorem of Kemperman~\cite{kemperman1974shannon} implies that if we take any sequence
	$\pi_k$ of priors attaining supremum in~\eqref{eq:i2}, the sequence of corresponding induced distributions
	$P_{Y_k}$ converges, in the sense that $D(P_{Y_k} \| Q^*)\to0$ (and hence in total variation), with limit $Q^*$
	independent of the sequence $\pi_k$. This unique $Q^*$ is what also achieves optimality of $F(\Phi,\Theta)$.

	Unfortunately, the argument above is very informal. The function $P \mapsto c(P)$ may not be measurable, the
	stochastic transformation taking element $P \in \Phi$ and outputting a random element $Y\sim P$ may not be a Markov
	kernel, etc. We proceed, thus, in a rather different way.

\begin{lemma}\label{lem:ddif} Let $P,Q,R\ll \mu$ and $f_P,f_Q,f_R$ denote their densities. Then (with $\log$ of the
ratio evaluated according to~\eqref{eq:log_inf})
	\begin{equation}\label{eq:ddif}
		\EE_{P} \left[\log {f_R\over f_Q}\right] = D(P\|Q) - D(P\|R)\,,
\end{equation}	
	whenever not both divergences are infinite. 
\end{lemma}
\begin{proof}
	First, suppose $D(P\|Q)=\infty$ and $D(P\|R)<\infty$. Then that $P[f_{R}(Y)=0] = 0$, and hence in computation of
	the expectation in~\eqref{eq:ddif} only 
	the second part of convention~\eqref{eq:log_inf} can possibly apply. Since also $f_P>0$
	$P$-almost surely, we have 
	\begin{equation}\label{eq:i7}
			\log {f_{R} \over f_Q} = \log {f_{R} \over f_P} + \log {f_P \over f_Q}\,,
	\end{equation}
	with both $\log$'s evaluated according to~\eqref{eq:log_inf}.
	Taking expectation over $P$ we see that the first term, equal to $-D(P\|R)$, is finite, whereas the second term
	is infinite. Thus, the expectation in~\eqref{eq:ddif} is well-defined and equal to $+\infty$, as is the LHS
	of~\eqref{eq:ddif}.

	Now consider $D(P\|Q)<\infty$. This implies that $P[q(Y)=0]=0$ and
	this time in~\eqref{eq:ddif} only the first part of convention~\eqref{eq:log_inf} can apply. Thus, again we have
	identity~\eqref{eq:i7}. Since the $P$-expectation of the second
	term is finite, and of the first term non-negative, we again conclude that expectation in~\eqref{eq:ddif} is
	well-defined, equals the LHS of~\eqref{eq:ddif} (and both sides are possibly equal to $-\infty$).
\end{proof}

\begin{lemma}\label{lem:lbfin} Let $P,Q\ll \mu$ and $f_P,f_Q$ be their relative densities. Then (with
convention~\eqref{eq:log_inf} for the log) we have
	\begin{equation}\label{eq:lbfin}
		\EE_P \left[ \max(\log {f_P\over f_Q}, 0)\right] \ge -{\log e \over e}\,.
\end{equation}	
	Consequently, the expectation $\EE_P\left[ \log {f_P\over f_Q}\right]$ is well-defined and
	non-negative (but could be $+\infty$).
\end{lemma}
\begin{proof}
	Let $g(x) = \max(x \log x, 0)$. It is clear $-{\log e\over } \le g(x) \le 0$ for all $x$. Since $f_P(Y)>0$ for
	$P$-almost all $Y$, in 
	convention~\eqref{eq:log_inf} only the $1\over 0$ case is possible, which is excluded by the $\max(\cdot,0)$
	from the expectation in~\eqref{eq:lbfin}. Thus, the LHS in~\eqref{eq:lbfin} equals
		\begin{align*} \int_{\{f_P>f_Q>0\}} f_P(y) \log {f_P(y)\over f_Q(y)} d\mu &= \int_{\{f_P>f_Q>0\}} f_Q(y)
		{f_P(y)\over f_Q(y)} \log {f_P(y)\over f_Q(y)} d\mu \\
			&= \int_{\{f_Q>0\}} g\left({f_P(y)\over
		f_Q(y)}\right) d\mu \\
		&\ge -{\log e\over e}\,.
\end{align*}		
	Since the negative part of $\EE_P\left[ \log {f_P\over f_Q}\right]$ is bounded, the expectation is well-defined.
	If $P[f_Q=0]>0$ then it is clearly $+\infty$. Otherwise, the said expectation equals $D(P\|Q)\ge0$.
\end{proof}
	Below we will freely use two facts about well-defined integrals (and expectations). If $f=g+h$ and $\int |h|
	d\mu < \infty$ the $\int f d\mu$ and $\int g d\mu$ are defined or undefined simultaneously. If $\mu =
	\mu_0+\mu_1$ and $\int fd\mu$ is well-defined, then so are $\int f d\mu_0$ and $\int fd\mu_1$ (the opposite is
	note true, since we can have $\int f d\mu_0 = +\infty$ and $\int f d\mu_1=-\infty$).

	Let us introduce a collection of distributions $\Pi$ on $\Theta \times \calY$ as follows:
	$$ \Pi = \left\{\sum_{k=1}^m p_k \delta_{\theta_k} \times P_k: p_k \ge 0, \sum p_k = 1, \theta_k \in \Theta, P_k \in
	\Phi\right\}\,.$$
	Note that $\Pi$ is convex. We give $\Pi$ any topology under which linear operations are continuous (e.g.
	topology of total variation).
	For each element $\pi \in \Pi$ we denote by $\pi_Y = \sum_{k=1}^m p_k P_k$ the marginal induced on the second
	coordinate. By constraint~\eqref{eq:dist_th} we have $\pi_Y \ll \mu$ and thus we denote $f_\pi = {d\pi_Y \over
	d\mu}$ the relative density of $\pi_Y$. On $\Pi$, let us define the following functional:
	\begin{equation}\label{eq:tilj_def}
		\tilde J(\pi) = \begin{cases}
				\EE_{\pi} \left[\log {f_{\theta}(Y)\over f_\pi(Y)}\right]\,, &\EE[\cdot]\mbox{~is well-defined}\\
				-\infty, &\mbox{o/w}.
		\end{cases}
\end{equation}	
	We remind of the convention~\eqref{eq:log_inf} for the log, and observe that $f_\pi(Y)>0$ almost surely, implying that
	only the first clause of the convention can possibly apply.

	%
\def\dom{\mathrm{dom}\,}
	
	Let $\dom \tilde J = \{\pi: \tilde J(\pi) > -\infty\}$.

\begin{lemma}[Properties of $\tilde J$] The set $\dom \tilde J$ is convex. 
	The functional $\pi \mapsto \tilde J(\pi)$ is concave. Furthermore, $\tilde J$ satisfies for all $\pi_0,\pi_1 \in \dom \tilde J$ and
	$\lambda \in [0,1]$ the bound
		\begin{equation}\label{eq:j_ub}
			\tilde J(\lambda \pi_1 + (1-\lambda) \pi_0) \le 
			\lambda \tilde J(\pi_1) + (1-\lambda) \tilde J(\pi_0)  + h(\lambda)\,,
\end{equation}			
	where $h(x) = -x \log x - (1-x) \log(1-x)$ is entropy of $\Ber(x)$ random variable.
\end{lemma}
\begin{proof}

Indeed, consider $\pi = \lambda \pi_1 + (1-\lambda) \pi_0$. Then
	$\pi_0$-almost surely we have $f_{\pi_0}(Y) > 0$ and $f_{\pi}(Y)>0$. Thus, even under convention~\eqref{eq:log_inf}
	(first clause) we have $\pi_0$-almost surely:
	$$ \log {f_\theta(Y)\over f_{\pi}(Y)} =\log {f_\theta(Y)\over f_{\pi_0}(Y)}  + \log {f_{\pi_0}(Y)\over
	f_{\pi}(Y)}\,.$$
	Consequently, we get
	\begin{equation}\label{eq:i3}
		\EE_{\pi_0} \left[\log {f_\theta(Y)\over f_{\pi}(Y)}\right] =\EE_{\pi_0} \left[\log
	{f_\theta(Y)\over f_{\pi_0}(Y)} + \log {f_{\pi_0}(Y) \over f_{\pi}}\right]  = \tilde J(\pi_0) +
	D(\pi_{0,Y}\|{\pi_Y})\,.
\end{equation}	
	Since $D(\pi_{0,Y}\|{\pi_Y}) \le \log {1\over 1-\lambda}$ we conclude that the expectation in the LHS of the
	last display is well-defined and $>-\infty$. Similarly, the expectation over $\pi_1$ is also well-defined and
	$>-\infty$. Since $\EE_\pi = \lambda \EE_{\pi_1} + (1-\lambda)\EE_{\pi_0}$, we conclude that $\pi \in \dom
	\tilde J$.
 
	Next we prove concavity of $\tilde J(\cdot)$. Indeed, from non-negativity of KL divergence and
	identity~\eqref{eq:i3} we conclude
	$$ \tilde J(\lambda \pi_1 + (1-\lambda) \pi_0) \ge \lambda \tilde J(\pi_1) + (1-\lambda) \tilde J(\pi_0)\,.$$

	To prove the last claim, consider $\pi=\lambda \pi_1 + (1-\lambda)\pi_0$ and, explicitly, $\pi = \sum_{k=1}^m
	p_k \delta_{\theta_k} \times P_k$. This implies that 
	$$ \pi_u = \sum_k p_{u,k} P_k \times \delta_{\theta_k}\,, \qquad u\in \{0,1\}$$
	where $\lambda p_{1,k} + (1-\lambda) p_{0,k} = p_k$ for all $k\in [m]$. Next,
	define a joint distribution on four random variables: for all $k\in [m], b\in \{0,1\}, y \in \calY$ set
\begin{equation}\label{eq:}
			\PP[B=b, \theta = \theta_k, \phi=k, Y \in dy] = (1-\lambda - (1-2\lambda)b) p_{b,k} P_k(dy)\,.
\end{equation}		
	If $\pi \in \dom \tilde J$ then we have
\begin{equation}\label{eq:tilj_mi}
			\tilde J(\pi) = I(\phi; Y) - \EE[c(\phi,\theta)]\,,
\end{equation}		
	where $c(\phi,\theta) = D(P_\phi \| P_\theta) \in [0,+\infty]$ and both terms are finite.  Indeed, we have
		$$ \tilde J(\pi) = \sum_k p_k \EE_{Y\sim P_k}\left[ \log {f_{\theta_k}(Y)\over f_\pi(Y)} \right]\,.$$
	Now, let $f_{k} = {dP_k \over d\mu}$. Clearly, $P_k[f_k(Y) = 0] = P_k[f_\pi(Y)=0] = 0$ and thus, we have (under
	convention~\eqref{eq:log_inf} for all logs) $P_k$-almost surely
		\begin{equation}\label{eq:i12}
			\log {f_{\theta_k}(Y)\over f_\pi(Y)} = \log {f_{\theta_k}(Y)\over f_k(Y)} + \log {f_k(Y) \over
		f_\pi(Y)}\,.
\end{equation}		
	From Lemma~\ref{lem:lbfin} the $P_k$-expectation of each term is well-defined, and since $\log {f_k\over f_\pi}
	\le \log {1\over p_k}$, the second expectation is finite. Thus, overall we can take $P_k$-expectation
	of~\eqref{eq:i12} and conclude 
		$$ \EE_{Y\sim P_k} \left[\log {f_{\theta_k}(Y)\over f_\pi(Y)}\right] = D(P_k \| \pi_Y) - 
			D(P_k \| P_{\theta_k})\,.$$
	Summing over $k$ we obtain~\eqref{eq:tilj_mi}. Note that $I(\phi;Y) \le \log m$ and thus for any $\pi \in \dom
	\tilde J$ we must have $D(P_k\|P_{\theta_k}) < \infty$ for any $k$ with $p_k>0$.

	Similarly, we show that $\tilde J(\pi_b) = I(\phi; Y|B=b) - \EE[c(\phi,\theta)|B=b]$ for $b=0,1$. Finally, since
	$\phi \dperp B|Y$ and the chain rule for mutual information we obtain
	$$ I(\phi; Y) = I(\phi, B; Y) = I(\phi; Y|B) + I(B;Y) \le I(\phi; Y|B) + h(\lambda)\,,$$
	where in the last step we used the fact that $B \sim \Ber(\lambda)$.
\end{proof}

	We next define the following functional: For every $\pi \in \Pi$ and $q:\calY\to\mreals_+$ with $\int qd\mu = 1$ 
	we define 
	$$ J_1(\pi, q) = \begin{cases}
				\EE_{\pi} \left[\log {f_{\theta}(Y)\over q(Y)}\right]\,, &\EE[\cdot]\mbox{~is well-defined}\\
				+\infty, &\mbox{o/w}.
		\end{cases}
	$$
	Again, we remind of the convention~\eqref{eq:log_inf}. We claim that
	\begin{equation}\label{eq:i5}
		F(\Theta, \Phi) = \inf_q \sup_{\pi \in \dom \tilde J} J_1(\pi,q)\,.
\end{equation}	
	Indeed, by agreement made in Appendix~\ref{app:tech_rem} for any inadmissible $q$ we have that its regret is declared 
	to be $+\infty$. Recall that $q$ is inadmissible if for any pair $\theta\in\Theta$ and $P\in \Phi$ the expectation $\EE_{Y\sim P} \log
	{f_\theta(Y)\over q(Y)}$ is undefined. But then taking $\pi =\delta_\theta \times P$, shows that for such a $q$
	we must have $\sup_{\pi \in \dom \tilde J} J_1(\pi,q) = +\infty$ as well. For admissible $q$'s it is clear that
	only the regular case in the definition of $J_1(\pi,q)$ applies and thus $\pi \mapsto J_1(\pi,q)$ is affine on
	$\Pi$. Consequently, supremum over $\pi\in \Pi$ coincides with the supremum over special 
	$\pi=\delta_\theta \times P$. 
	Note that $\pi = \delta_\theta \times P \in \dom \tilde J$ iff $D(P\|P_\theta) <\infty$. Therefore, proof
	of~\eqref{eq:i5} is completed once we can show 
		\begin{equation}\label{eq:i4}
			\sup_{\theta \in \Theta, P\in \Phi} \EE_P[\log {f_\theta\over q}] = \sup_{\theta \in \Theta, P\in \Phi:
		D(P\|P_\theta)<\infty} \EE_P[\log {f_\theta\over q}]\,. 
\end{equation}		
	To show~\eqref{eq:i4}, we first prove the following lemma:

	Equipped with Lemma~\ref{lem:ddif}, we can argue~\eqref{eq:i4} as follows. 
	Suppose for some $P\in \Phi$ we have $D(P\|Q)=\infty$, where $dQ = q d\mu$. Bu then, by
	assumption~\eqref{eq:dist_th} there exists $\theta_0$ such that $D(P\|P_{\theta_0})<\infty$.
	Thus~\eqref{eq:ddif} implies both sides of~\eqref{eq:i4} evaluate to $+\infty$. Next, suppose for every $P\in
	\Phi$ we have $D(P\|Q) < \infty$. Then again from~\eqref{eq:ddif} we see that pair $(P,\theta_0)$ with
	$D(P\|P_{\theta_0})=\infty$ yield $-\infty$ values in the RHS of~\eqref{eq:i4} and can be excluded.

	\begin{lemma} For any $\pi \in \dom \tilde J$ and any $q$ we have
		\begin{equation}\label{eq:i8}
			J_1(\pi, q) = \tilde J(\pi) + D(\pi_Y \| Q) > -\infty\,.
\end{equation}		
		In particular, for any $\pi \in \dom \tilde J$ we have
			$$ \tilde J(\pi) = \min_q J_1(\pi,q) = J_1(\pi,f_\pi)\,.$$
	\end{lemma}
	\begin{proof}
		As before, under $\pi$ measure on $(\theta,Y)$ we have that $f_\pi(Y)>0$ almost surely. 
		Since $\tilde J(\pi) > -\infty$, we see from~\eqref{eq:tilj_def} that $\pi[f_\theta(Y) = 0] = 0$ and
		thus $f_\theta(Y)>0$ almost surely as well. Thus, under convention~\eqref{eq:log_inf} for all logs we
		have almost surely 
			$$ X \eqdef \log {f_\theta(Y)\over q(Y)} = \log {f_\theta(Y)\over f_\pi(Y)} + 
				\log {f_\pi(Y)\over q(Y)}\,.$$
		Denoting the two terms as $A$ and $B$. By assumption $\EE_\pi[A]=\tilde J(\pi) > -\infty$. On the other
		hand, from Lemma~\ref{lem:lbfin} we know $\EE_\pi[\max(B,0)]>-\infty$
		and thus the expectation $\EE_\pi[X]$ is well-defined (and $>-\infty$). Consequently, the value
		$J_1(\pi, q) = \EE_\pi[X] = \EE_\pi[A] + \EE_\pi[B]$, completing the proof of~\eqref{eq:i8}.
	\end{proof}

	We next establish the saddle-point property of $J_1(\pi,q)$ on finite-dimensional subsets of $\Pi$.
	\begin{lemma}[Saddle point] \label{lem:saddle} Let $\pi_1,\ldots,\pi_k$ be any elements in $\dom \tilde J$ with
	$\max_i \tilde J(\pi_i) < \infty$. Let $\Pi_k =
	\mathrm{co}(\pi_1,\ldots,\pi_k)$. Then the function $\pi \mapsto \tilde J(\pi)$ is continuous on $\Pi_k$ and
	achieves its maximum $F=\max_{\pi \in \Pi_k} \tilde J(\pi)$. 
	For any maximizer $\pi^*$, set $q^* = f_{\pi^*}$. Then the pair $(\pi^*,q^*)$ is the saddle point: 
	For all $\pi \in \Pi_k$ and all $q\ge 0$,
	$\int q d\mu = 1$ we have
	\begin{equation}\label{eq:saddle}
			J_1(\pi, q^*) \le F =  J_1(\pi^*, q^*) \le J_1(\pi^*, q)
	\end{equation}		
		The density $q^*$ with the property $F=\max_\pi J_1(\pi,q^*)$ is unique. 
	\end{lemma}
	\begin{proof} Let us take $q_0 = {1\over k} \sum_{i=1}^k f_{\pi_i}$ and denote $dQ_0 = q_0 d\mu$. Then
	from~\eqref{eq:i8} we have
		$$\tilde J(\pi) = J_1(\pi, q_0) - D(\pi_Y\|Q_0)\,.$$
		Notice that the map $\pi \mapsto D(\pi_Y \|Q_0)$ is continuous on $\Pi_k$ (since $0\le {d\pi_Y\over dQ_0} \le k$,
		so bounded-convergence theorem holds). On the other hand, $\pi \to J_1(\pi,q_0)$ is affine and
		continuous on $\Pi_k$ (since it is finite at extremal points $\pi_1,\ldots\pi_k$). This proves
		continuity of $\tilde J(\pi)$ on $\Pi_k$. (Note that continuity on the interior of $\Pi_k$ automatically follows
		from concavity.)

		The right-hand inequality in~\eqref{eq:saddle} follows from~\eqref{eq:i8} and non-negativity of
		divergence. For the (key) left-most inequality fix $\pi \in \Pi_k$ and define for each $\lambda \in [0,1)$ the density
		$q_\lambda = (1-\lambda) q^* + \lambda f_\pi$ and $\pi_\lambda = (1-\lambda)\pi^* + \lambda \pi$. Then, we have
		$$ F \ge \tilde J(\pi_\lambda) = J_1(\pi_\lambda, q_\lambda) \ge
			(1-\lambda) J_1(\pi^*, q_\lambda) + \lambda J_1(\pi, q_\lambda)\\
			\ge (1-\lambda) F+ \lambda J_1(\pi, q_\lambda)\,.$$
		This implies, $J_1(\pi, q_\lambda) \le F$ for all $\lambda > 0$. From~\eqref{eq:i8} we have:
		$$ F \ge J_1(\pi, q_\lambda) = \tilde J(\pi) + D(\pi_Y\|Q_\lambda)\,.$$
		Taking limit as $\lambda \to 0$ and using lower-semicontinuity of divergence
		$$ \lim_{\lambda \to 0} D(\pi_Y\|Q_\lambda) \ge D(\pi_Y\|\pi_Y^*) $$
		results in
		$$ F \ge \tilde J(\pi) + D(\pi_Y\|\pi_Y^*) = J_1(\pi, q^*)\,,$$
		where the last step is by~\eqref{eq:i8}. This completes the proof of~\eqref{eq:saddle}.

		To prove uniqueness of $q^*$ suppose there is $\tilde q^*$ (density) and $\tilde Q^*$ (measure) such
		that $\sup_\pi J_1(\pi, \tilde q^*) = F$. Plugin $\pi = \pi^*$ in this identity and observe:
		$$ F \ge J_1(\pi^*, \tilde q^*) = \tilde J(\pi^*) + D(\pi_Y^* \| \tilde Q^*) = F + D(\pi_Y^*\|\tilde Q^*)
		\,,$$
		implying $\pi_Y^* = \tilde Q^*$.
	\end{proof}

	With these preparations we proceed to the main subject of this section. 
	\begin{proof}[Proof of Theorem~\ref{th:exist}]
		Identity~\eqref{eq:i5} implies that
			$$ F(\Theta,\Phi) = \inf_q \sup_{\pi \in \dom \tilde J} J_1(\pi,q)\,.$$
			From~\eqref{eq:i8} we have then
			\begin{equation}\label{eq:i11}
				\tilde F \eqdef \sup_{\pi \in \dom \tilde J} \tilde J(\pi) = \sup_\pi \inf_q J_1(\pi,q) \le \inf_q \sup_\pi
			J_1(\pi,q) = F(\Theta,\Phi) < \infty\,.
\end{equation}			
			Thus, consider any sequence $\pi_k'$ such that $\tilde J(\pi_k')\upto \tilde F$. Let us
			now denote $\Pi_k = \co(\pi'_1,\ldots,\pi'_k)$ and choose 
			$$\pi_k \in \argmax_{\pi \in \Pi_k} \tilde J(\pi)\,.$$
			Denote the sequence of induced densities $q_k \eqdef f_{\pi_k}$. (Lemma~\ref{lem:saddle} shows
			such $\pi_k$'s exist and $q_k$ only depends on $\Pi_k$ but not $\pi_k$.)
			We will prove the following facts:
			\begin{enumerate}
			\item The sequence of densities $q_k$ converges to a density $q^*$ in the sense of
			$D(Q_k\|Q^*) \to 0$ (and, thus, in total variation).
			\item 
				For every $\pi \in \Pi_k$ we have $J_1(\pi, q^*) \le \tilde F$.
			\item $\tilde F = F(\Theta,\Phi)$ and, furthermore, 
				\begin{equation}\label{eq:i9}
					J_1(\pi,q^*) \le F(\Theta,\Phi) \qquad \forall\pi \in \dom \tilde J\,.
			\end{equation}				
			\end{enumerate}
			This proves a convenient characterization (analog of~\eqref{eq:capred})
			\begin{equation}\label{eq:f_capred}
					F(\Theta,\Phi) = \sup_{\pi} \tilde J(\pi) 
			\end{equation}				
			and completes the proof of the Theorem. Indeed, from~\eqref{eq:i9} taking supremum over
			$\pi$ we obtain optimality of $q^*$. Had there existed another $\tilde q^*$ with the
			property~\eqref{eq:i9} then we would have from~\eqref{eq:i8}
				$$ \tilde J(\pi_k) \le \tilde J(\pi_k) + D(Q_k\|\tilde Q^*) = J_1(\pi,\tilde q^*) \le
				F(\Theta,\Phi)\,.$$
				Since $\tilde J(\pi_k)\upto F(\Theta,\Phi)$ we conclude that $D(Q_k\|\tilde Q^*)\to 0$,
				and thus $Q_k$ converges to $\tilde Q^*$ in total variation. But $Q_k$ converges to
				$Q^*$ as well, so $\tilde Q^*=Q^*$ and $\tilde q^*=q^*$ $\mu$-almost everywhere.

			To prove the first statement, we apply Lemma~\ref{lem:saddle} as follows. Let $F_m = \tilde
			J(\pi_m)$ and notice since $\pi_k \in \Pi_{k+m}$ for $m\ge 0$ that
				$$ F_k + D(Q_k \|Q_{k+m}) = J_1(\pi_k, q_{k+m}^*) \le F_{k+m} \le \tilde F\,.$$
				Thus, $D(Q_k\|Q_{k+m})\le \tilde F - F_k$ and $\sup_m D(Q_k\|Q_{k+m}) \to 0$ as
				$k\to\infty$. 
				This implies that $Q_k$ form a Cauchy sequence in total variation and thus have a limit
				point $Q^*$. From lower semicontinuity of divergence we also have $D(Q_k\|Q^*) \le
				\lim_{m\to\infty} D(Q_k \|Q_{k+m}) \le \tilde F - F_k$ and thus $D(Q_k \|Q^*) \to 0$ as
				$k\to \infty$.

			To prove the second claim, note that by Lemma~\ref{lem:saddle} for any $\pi \in \Pi_k$ we have
				$$ J_1(\pi, q_m) \le F_m \le \tilde F \qquad \forall m\ge k\,.$$
				On the other hand, $J_1(\pi,q_m) = \tilde J(\pi) + D(\pi_Y\|Q_m)$ and taking
				$m\to\infty$ and applying lower semicontinuity yet again, we get
				\begin{equation}\label{eq:i10}
					J_1(\pi, q^*) =  J(\pi) + D(\pi_Y\|Q^*) \le \tilde F\,.
				\end{equation}				

			Finally, to prove~\eqref{eq:i9} for an arbitrary $\pi$ (not necessarily $\in \cup_k \Pi_k$), we
			can simply reapply the previous argument with $\widehat{\Pi}_k \eqdef \co(\pi, \pi_1,\ldots,\pi_k)
			\supset \Pi_k$, to obtain sequence $\hat q_k \to \hat q^*$. Since $\pi_k \in \Pi_k$, for this
			new density we have
				$$ J_1(\pi_k, \hat q^*) \le \tilde F\,.$$
			But $J_1(\pi_k,\hat q^*) = \tilde J(\pi_k) + D(Q_k \| \widehat{Q}^*) = F_k + D(Q_k \|
			\widehat{Q}^*)$, implying $D(Q_k \| \widehat{Q}^*) \to 0$ and, thus, $Q^* = \widehat{Q}^*$, and
			in particular,~\eqref{eq:i10} holds. Taking supremum over $\pi$ in~\eqref{eq:i10} and comparing
			with~\eqref{eq:i11} we get $\tilde F = F(\Theta,\Phi)$. This establishes~\eqref{eq:i9}.
	\end{proof}

\section{Proof of Theorem~\ref{th:small}}\label{sec:small}

First, we show the  following result.
\begin{lemma} Let $\Phi$ be such that $C_1(\Phi) < \infty$ and every $P\in \Phi$ satisfies $P\ll \mu$. Let $\Phi_0 = \{P\in \Phi: D(P\|\Theta) < \infty\}$. Then
	$$ F(\Theta,\Phi) = F(\Theta,\Phi_0) $$
\end{lemma}
\begin{proof}
	We only need to show $F(\Theta,\Phi) \le F(\Theta,\Phi_0)$. To that end, fix $\epsilon>0$ and 
	consider any $q$ such that
		$$ \sup_{P\in \Phi} \sup_{\theta \in \Theta} \EE_P[\log {f_\theta\over q}] \le F(\Theta,\Phi_0) +
		\epsilon\,.  $$
	In addition, denote $q_1 = {dQ_1\over d\mu}$ -- density of the (unique) distribution $Q_1$ attaining the minimum
		$$ \min_{Q_1} \sup_{P\in \Phi} D(P\|Q_1) = C_1(\Phi) < \infty\,.$$
	Similarly, for any $P\in \Phi$ we denote by $f_P = {dP\over d\mu}$. For any $\theta$ and $P$ we have
		$$ \EE_P[ \log {f_\theta\over q_1}] = D(P\|Q_1) - D(P\|P_\theta)\,.$$
	Indeed, almost surely (with convention~\eqref{eq:log_inf}) we have $\log {f_\theta\over q_1} = \log {f_P\over
	q_1} - \log{f_P\over f_\theta}$. Denoting $q_\lambda = \lambda q_1 + (1-\lambda)
	q$, we have
		$$ \EE_P[ \log {f_\theta\over q_\lambda}] \le \log {1\over \lambda} + \EE_P[ \log {f_\theta\over q_1}] =
		\log {1\over \lambda} D(P\|Q_1) - D(P\|P_\theta)\,.$$
	Thus, for any $P \not \in \Phi_0$ the above evaluates to $-\infty$. On the other hand, we have
		$$ \EE_P[ \log {f_\theta\over q_\lambda}] \le \log {1\over 1-\lambda} + \EE_P[ \log {f_\theta\over q}]
		$$
	And thus taking supremum over $P\in \Phi$ and $\theta \in \Theta$ we get
		$$ \sup_{P \in \Phi} \sup_\theta \EE_P[ \log {f_\theta\over q_\lambda}]  \le \begin{cases} -\infty, &
		P\not \in \Phi_0\\
			F(\Theta,\Phi_0) + \epsilon + \log {1\over 1-\lambda}, & P\in \Phi_0
		\end{cases} $$
	Taking $\lambda,\epsilon \to 0$ completes the proof.
\end{proof}

The Theorem follows as a special case of the following result.
\begin{lemma} Let $\Phi$ be such that a) $P_\theta \in \Phi$ for all $\theta\in\Theta$, b) $P\ll \mu$ for every $P\in
\Phi$, c) $C_1(\Phi) < \infty$. Then, we have for any $\epsilon>0$ such that $\lambda_0 = {C_1(\Phi)\over \epsilon} < 1$
	$$ F(\Theta,\Phi) \le F(\Theta, \Theta_\epsilon) + {h(\lambda_0)\over 1-\lambda_0}\,.$$
\end{lemma}
\begin{proof} By preceding Lemma, we can assume further that $D(P\|\Theta)<\infty$ for all $P\in \Phi$. Therefore,
Theorem~\ref{th:exist} applies, and in particular (see~\eqref{eq:f_capred}):
	$$ F(\Theta,\Phi) = \sup_{\pi} \tilde J(\pi)\,. $$
	Since $\Theta \subset \Phi$, we have that $F(\Theta,\Phi) \ge C(\Theta)\ge 0$. Thus, there exist $\pi$ such that
	$\tilde J(\pi) \ge 0$, and for any such $\pi$ we have (see~\eqref{eq:tilj_mi})
	$$ 0 \le \tilde J(\pi) = I(\phi; Y) - \EE[c(\phi,\theta)] \le C_1(\Phi) - \EE[c(\phi, \theta)]\,.$$
	This implies via Markov inequality that
	$$ \PP[c(\phi,\theta) > \epsilon] \le {C_1(\Phi)\over \epsilon}\,.$$
	This means that we can represent $\pi = \lambda \pi_1 + (1-\lambda) \pi_0$ with $\pi_0 = \sum_{k} p_{0,k}
	P_k\times \delta_{\theta_k}$ and $P_k \in \Theta_\epsilon$ whenever $p_{0,k}>0$. Furthermore, $\lambda \le
	{C_1(\Phi)\over \epsilon}$. By the bound~\eqref{eq:j_ub} we have
		$$ \tilde J(\pi) \le h(\lambda) + \lambda \tilde J(\pi_1) + (1-\lambda) \tilde J(\pi_0) \le h(\lambda) +
		\lambda F(\Theta,\Phi) + (1-\lambda) F(\Theta, \Theta_\epsilon)\,. $$
	Taking supremum over $\pi$ we obtain (after rearranging terms)
		$$ F(\Theta,\Phi) \le F(\Theta, \Theta_\epsilon) + {h(\lambda)\over 1-\lambda}\,.$$
	The proof is completed by noticing that $x\mapsto {h(x)\over 1-x}$ is increasing on $[0,1)$.
\end{proof}

\begin{proof}[Proof of Theorem~\ref{th:small}]
If we apply previous theorem with $\Phi$ replaced by $\Phi^{\otimes n}$ and $\epsilon$ replaced by $n\epsilon$ we
obtain:
	$$ F_n(\Theta, \Phi^{\otimes n}) \le F_n(\Theta, \Theta_{\epsilon_n}^{\otimes n}) + {h(\lambda_0) \over
	1-\lambda_0} $$
	where $\lambda_0 = {C_n(\Phi)\over n\epsilon} = {\tau_n\over \epsilon}$. If $\epsilon = \epsilon_n \gg \tau_n$
	then $\lambda_0\to0$ and the proof is complete.
\end{proof}

\subsection{On difference between $C_n(\Theta)$ and $C_n(\Theta_\epsilon)$}\label{sec:cex_cap}

As we discussed in~\eqref{eq:small_disc}, the meaning of Theorem~\ref{th:small} is to sandwich the misspecified
regret between two well-specified ones: $C_n(\Theta)$ and $C_n(\Theta_\epsilon)$. In this section, we demonstrate 
by a simple example that the growth rates of $C_n(\Theta)$ and $C_n(\Theta_\epsilon)$ could be very different.


We consider an extension of the Gaussian location model, where a single measurement $Y_1$ is
produced from $\phi = (\phi_0,\ldots)$ as
	$$ Y_1 = \phi + Z, \qquad Z \sim \mathcal{N}(0, I_{\infty})\,.$$
	(i.e. each of $Y_1, Y_2, \ldots, Y_n$ is itself an infinite sequence). We will take as $\Phi$ the Hilbert brick:
	$$ \Phi = \{\phi: 0 \le \phi_j\le 2^{-j}, j=0,\ldots\} = \prod_{j=0}^\infty [0,2^{-j}]\,.$$
	(It is known that $\Phi$ is compact in $\ell_2$, for example). Let $\Theta = \{\phi: \phi_1=\phi_2 =
	\cdots = 0\}$.

	Denote by $C_n([0,a])$ capacity of the 1D-GLM model. Then, from~\eqref{eq:glm_asymp} we have 
	$$ C_n(\Theta)=C_n([0,1]) = {1\over 2} \ln (n/(2\pi)) + o(1)$$

	For $\Phi$ we have
	$$ C_n(\Phi) = \sum_{j=0}^\infty C_n([0,2^{-j}])\,.$$
	 Indeed, since the conditional transformation
	$\theta \mapsto Y$ has structure of a parallel memoryless channel, the optimization in~\eqref{eq:capred} can be
	reduced to $\pi$ that are independent across coordinates of $\theta$, cf.~\cite[Theorem
	5.1]{polyanskiy2014lecture}. (The fact that there are countably-infinite
	number of coordinates does not cause any complications due to a certain continuity of 
	mutual information: $I(A_1,\ldots;
	B_1,\ldots) = \lim_{n\to\infty} I(A^n; B^n)$.) 

	To compute $C_n(\Phi)$, we note that from a special case of~\eqref{eq:glm_shtar} (with $\Theta =[0,a] \subset \mreals$, $\tau=0$, $v=1$) we get
\begin{equation}\label{eq:glm_1d}
	\Gamma_n([0,a]) = \log \left(1+a \sqrt{n\over 2\pi}\right)\,.
\end{equation}
	Shtarkov~\eqref{eq:glm_1d} we get
		$$ C_n([0,a]) \le \Gamma_n([0,a]) = \log(a\sqrt{n\over 2\pi} + 1)\,.$$
	We can also show that $\sum_{j=0}^\infty \log(a_0 2^{-j} + 1) \asymp (\log a_0)^2 $ within absolute constants. 
	Thus, we have
		$$ C_n(\Phi) = O(\log^2  n)\,. $$
	(In fact, the argument below also shows $C_n(\Phi) = \Theta(\log^2 n)$.)

	Next, note that $C_n([0,a]) = C_1([0,\sqrt{n}a])$ and from the asymptotics we know that $C_1([0,a])=\ln (a/(2\pi)) +
	o(1)$, implying that we always have
		$$ C_1([0,a]) \ge \ln a - c_1 $$
	for some constant $c_1>0$. Thus, we conclude
		$$ C_n([0,a]) \ge \ln (\sqrt{nc_2}a)\,,$$
	for some $0<c_2<1$. 

	Observe that 
		$$ \sum_{j\ge k} (2^{-j})^2 = {3\over 4} 4^{-k}\,.$$
	Therefore, any $\Theta_{\epsilon}$ always contains a sub-brick:
		$$ [0,1] \times \{0\} \cdots \times \{0\} \times [0,2^{-k_1}] \times [0,2^{-k_1-1}] \cdots \,,$$
	where $k_1$ is minimal such that
		$$ {3\over 4} 4^{-k_1} \le {1\over 2}\epsilon\,.$$
	That is $k_1 = {1\over 2} \log_2 {1\over \epsilon} + O(1)$. 
	Thus, we get that
		\begin{align}C_n(\Theta_{\epsilon}) &\ge C_n(\Theta) + \sum_{i\ge k_1} C_n([0,2^{-i}]) \\
				&\ge C_n(\Theta) + \sum_{j\ge 0} \max(\ln(\sqrt{nc_2} 2^{-k_1-j}), 0) \\
				&\ge C_n(\Theta) + c (\ln (n \epsilon))^2\,,
		\end{align}				
	for some constant $c>0$. Overall, we see that for any $\epsilon>0$ the order of $C_n(\Theta) \asymp \log n$
	whereas $C_n(\Theta_\epsilon) \asymp \log^2 n$. 


\section{Open questions}\label{app:open}

\begin{itemize}
\item Are there examples where $C_n \ll F_n^{(PAC)} \asymp \Gamma_n < \infty$ ? What about $C_n \ll F_n^{(PAC)} \ll \Gamma_n$?
\item Assumptions in our Theorem~\ref{th:exist} unfortunately rule out the PAC case of $\Phi = \calP_{iid}(\calY^n)$. Can we extend it
to this case? Even in the special case of the setting of Theorem~\ref{th:glm}, can we prove existence and uniqueness 
of the minimizer?
\item In Theorem~\ref{th:glm} what is the order of the difference between $F_n^{(PAC)}$ and $C_n$?
\item In the context of Theorem~\ref{th:glm} can it be shown that no Bayes mixture is able to achieve optimal 
$F_n^{(PAC)} + o(1)$ regret? (Perhaps it can even be shown that no estimator with Gaussian tails can do so.)
\item Extend Theorem~\ref{th:glm} to (a) exponential families, (b) general smooth families.
\item \textbf{(closed)} Consider $\calY = \mathbb{Z}_+$ and $\{P_\theta\}$ to be the class of all distributions on $\calY$ with first
moment bounded by 1. It is easy to show that $C_n < \infty$, while $\Gamma_n = F_n^{(PAC)} = \infty$ (Shtarkov sum
is unbounded). 
We note, cf.~\cite{jia2021zplus}, that restricting $\Phi_n$ to a subset $\Phi_n = \{P^{\otimes n}:
D(P\|\Theta)<\infty\}$, still results in $F_n(\Theta,\Phi_n) = \infty$ (note that with this
restriction on the data generating distribution
the oracle loss is always finite, although unbounded). We also mention that for this model class
the results of~\cite{boucheron2008coding} show $C_n = \omega(n^\alpha)$ for any $\alpha < 1/2$,
and a more detailed analysis~\cite{jia2021zplus} shows $C_n = \tilde \Theta(\sqrt{n})$.

\apxonly{
\item \textcolor{blue}{\textbf{ADD:} Prove batch versions of various results in this paper!}
}
\end{itemize}

\end{document}